\documentclass[journal]{IEEEtran}
\usepackage{cite}
\usepackage{amsmath}
\usepackage{amssymb}
\usepackage{amsthm}
\usepackage{amsfonts}

\newtheorem{theorem}{Theorem}
\newtheorem*{theorem*}{Theorem}
\newtheorem*{corollary*}{Corollary}

\theoremstyle{remark}

\newtheorem*{claim*}{Claim}

\theoremstyle{plain}
\newtheorem{scenario}{Scenario}

\newtheorem{constraint}{PK}

\usepackage{bbm}
\usepackage{xfrac}
\usepackage[caption=false]{subfig}

\usepackage{textcomp}
\usepackage{xcolor}
\def\BibTeX{{\rm B\kern-.05em{\sc i\kern-.025em b}\kern-.08em
    T\kern-.1667em\lower.7ex\hbox{E}\kern-.125emX}}

\usepackage{graphicx} % Required for inserting images

\usepackage{tikz}
\usepackage{pgfplots}
\usetikzlibrary{calc}

\usetikzlibrary{matrix,positioning,quotes,fit}
\usetikzlibrary{backgrounds}
\usetikzlibrary {arrows.meta}
\usetikzlibrary{decorations.pathmorphing}
\usetikzlibrary{positioning,decorations.pathreplacing}
\pgfplotsset{compat=1.18}

\tikzset{standard/.style={matrix of nodes,left delimiter={(},right delimiter={)},inner sep=0pt,nodes={inner sep=0.3em}}}
\tikzset{submatrix/.style = {rectangle, rounded corners,
  fill=yellow, draw, inner sep=0pt}}

\usetikzlibrary{calc}
\usepackage{relsize}
\usepackage{scalefnt}

\tikzset{fontscale/.style = {font=\relsize{#1}}}

\usepackage[utf8]{inputenc}%added by XZ to have accent in the names like Barré ..
\usepackage{algorithm}
\usepackage{algpseudocode}
\usepackage{multirow}
\usepackage{tabularx,booktabs}
\usepackage{rotating}
\usepackage{textcomp}
\usepackage{xcolor}
\usepackage{enumerate}
\usepackage{enumitem} %for enumerate "wide" option
\usepackage{wrapfig}
\usepackage{tcolorbox}

\usepackage{times}
\usepackage{soul}
\usepackage{url}
\usepackage[hidelinks]{hyperref}
\usepackage[absolute,overlay]{textpos}% by XZ: for the text in the top margin..

\usepackage{nomencl}%XZ added for notations
\makenomenclature%XZ added to do notation table

\newcount\Comments  % 0 suppresses notes to selves in text
\usepackage{color}
\definecolor{darkgreen}{rgb}{0,0.5,0}
\definecolor{purple}{rgb}{1,0,1}
\newcommand{\kibitz}[2]{\ifnum\Comments=1\textcolor{#1}{#2}\fi}
\begin{document}

\title{The Impact of Operational-Data Fidelity when Assessing Safety-Critical Autonomous-Vehicle Software}
%\title{Conservative Reliability Assessment Under Failure-Type Uncertainty in Autonomous-Vehicle Software}
%\title{The Impact of Operational-Data Fidelity in Assessments of Safety-Critical Autonomous-Vehicle Software}

\author{Kizito~Salako, Rabiu~Tsoho~Muhammad
        % <-this % stops a space
\IEEEcompsocitemizethanks{\IEEEcompsocthanksitem K.~Salako and R.~Muhammad are with the Centre for Software Reliability, City St. George's, University of London, Northampton Square EC1V 0HB, U.K.
%\protect\\
(email: \{k.o.salako,rabiu-tsoho.muhammad\}@city.ac.uk)}%<-this % stops an unwanted space
%\thanks{Manuscript received February, 2026. Revised June, 2026. Resubmitted October, 2026.}
}

% \author{\IEEEauthorblockN{Kizito Salako\IEEEauthorrefmark{1}, Xingyu Zhao\IEEEauthorrefmark{2}\\}%
% \IEEEauthorblockA{\IEEEauthorrefmark{1}\textit{Centre for Software Reliability, City, University of London}, Northampton sq. EC1V 0HB, U.K.\\
% Email: k.o.salako@city.ac.uk\\}%
% \IEEEauthorblockA{\IEEEauthorrefmark{2}\textit{Department of Computer Science, University of Liverpool}, Ashton Street L69 3BX,  U.K.\\
% Email: xingyu.zhao@liverpool.ac.uk}%
% \thanks{..}}

%{Shell \MakeLowercase{\textit{et al.}}: Bare Demo of IEEEtran.cls for Computer Society Journals}

\IEEEtitleabstractindextext{%
\begin{abstract}
% For safety--critical software, data from the software's operational past (e.g. a sequence of success and failure events experienced by the software) can provide strong statistical support for reliability claims about the software. However, such data might not describe past software failure events in sufficient detail, and this might leave a reliability assessment (based on this data) unable to account for important features of past software failures. In this paper, by extending conservative Bayesian inference (CBI) techniques used in reliability assessment, we illustrate a principled statistical approach for checking the robustness of reliability claims derived from insufficiently detailed operational data. We demonstrate the extent to which insufficient detail in operational data can undermine software reliability claims in autonomous vehicle (AV) safety assessment scenarios. Reliability claims derived from insufficiently fine-grained data might be dangerously optimistic, despite a concerted effort by an assessor to use such data conservatively during the assessment. While these findings are consistent with previous work on the impact of \emph{statistical model fidelity} in Bayesian software reliability assessments, our work clarifies why attempts to use low-fidelity data conservatively can be na\"{\i}ve. To the best of our knowledge, we give the first conservative estimates of the impact of \emph{data fidelity} on reliability assessments.

For safety--critical software, operational data (e.g. sequences of software successes and failures) can provide strong statistical support for reliability claims. However, insufficient detail about past software failures may leave assessments unable to account for important features of failure behavior. In this paper, we extend conservative Bayesian inference (CBI) techniques for reliability assessment to check the robustness of reliability claims based on such data. We show how insufficient detail in operational data can undermine software reliability claims in autonomous vehicle (AV) safety assessment scenarios: even when used conservatively, low-fidelity data may yield dangerously optimistic conclusions. While these findings are consistent with previous work on the impact of \emph{statistical model fidelity} in Bayesian software reliability assessments, our work clarifies why attempts to use low-fidelity data conservatively can be na\"{\i}ve. To the best of our knowledge, we give the first conservative estimates of the impact of \emph{data fidelity} on reliability assessments.
\end{abstract}

\begin{IEEEkeywords}
software reliability assessment, conservative Bayesian inference, CBI, operational data fidelity, AV safety
\end{IEEEkeywords}
}
\maketitle
\IEEEdisplaynontitleabstractindextext

%\begin{textblock*}{20cm}(1cm,1cm)
%\textcolor{red}{Published at ISSRE2019: 
%\url{https://ieeexplore.ieee.org/document/8987509}}
%\end{textblock*}

%\IEEEpeerreviewmaketitle

%\IEEEraisesectionheading{
\section*{List of Abbreviations and Symbols}

\setlist[description]{%
  leftmargin=!,
  labelwidth=1.5cm, % adjust (e.g., 2.0–2.8cm) for alignment
  labelsep=0.6em,
  itemsep=0.2ex,
  topsep=0.3ex
}

{\small
\begin{description}
  %\item[ADS] automated driving system
  %\item[AEB] autonomous emergency braking
  %\item[AI] artificial intelligence
  %\item[ANSI/UL] American National Standards Institute / Underwriters Laboratories
  \item[ASIL] automotive safety integrity level
  \item[AUC] area under the curve
  \item[AV] autonomous vehicle
  \item[CAN] controller area network
  \item[CBI] conservative Bayesian inference
  \item[CEHSS] critical event-handling subsystem
  %\item[F1] F1 score; harmonic mean of precision and recall
  \item[FP, FN] false positive, false negative
  %\item[FPR, TPR] false-positive rate, true-positive rate
  \item[FTDSS] fault-tolerant decision subsystem
  \item[GNSS] global navigation satellite system
    \item[IMU] inertial measurement unit
  \item[HARA] hazard analysis and risk assessment
  \item[i.i.d.] independent and identically distributed
  %\item[ISO] International Organization for Standardization
  %\item[ISO/PAS] International Organization for Standardization/Publicly Available Specification
  %\item[ML] machine learning
  \item[LiDAR] light detection and ranging
  \item[ODD] operational design domain
  %\item[OOD] out of distribution
  \item[\emph{pfc}] probability of failure per classification
  \item[\textbf{PK}] evidence-based prior knowledge
  %\item[PR] precision--recall
  \item[ROC] receiver operating characteristic
  \item[RSS] responsibility-sensitive safety
  \item[SOTIF] safety of the intended functionality
  \item[SW] software
  %\item[TNR] true-negative rate
  \item[V\&V, V2V] verification and validation, vehicle-to-vehicle

  \item[$P,Q$] unknown probabilities of FP and FN classifications
  %\item[$Q$] unknown probability of FN classification
  \item[$p,q$] generic values of the random quantities $P,Q$
  \item[$\Theta$] unknown probability of correct classification
  \item[$\boldsymbol\Omega$] sample space of feasible $(P,Q)$ pairs
  \item[$\mathcal D$] set of admissible prior distributions satisfying the assessor's PK constraints
  \item[$S_i,S_1,S_2$] regions in the partition of $\Omega$ induced by the \textbf{PK}s

  \item[$n$] number of observed classifications
  \item[$k_1,k_2$] number of observed FPs and FNs
  %\item[$k_2$] number of observed FNs
  \item[$k$] total number of observed failures
  \item[$n_1$] number of successes before first failure
  \item[$n_2$] number of successes needed after first failure

  \item[$a$] assessor's prior confidence that the classifier meets the engineering target
  \item[$b$] claimed upper bound on \emph{pfc}; %equivalently $1-b$ is the minimum required accuracy
  \item[$b_1$] engineering target lower bound on classifier accuracy
  %\item[$1-b_1$] Engineering target upper bound on \emph{pfc}
  \item[$l$] lower confidence--bound on \emph{pfc}
  \item[$l_1,l_2$] practical lower limits on $P$ and $Q$
  %\item[$l_2$] practical lower limit on $Q$

  %\item[$L(n,k_1,k_2;p,q)$] likelihood for typed FP/FN failure data
  %\item[$L(n,k;p,q)$] likelihood for untyped/coarse failure-count data
  \item[$L_i$] likelihood value at $(p_i,q_i)\in S_i$
  \item[$L_i^*, L_{i*}$] extrema of $L_i$ over the relevant regions of $S_i$
  \item[$p_i,q_i$] support-point coordinates of a worst-case prior
  \item[$\Phi^*$] conservative posterior-confidence infimum / worst-case posterior confidence
  %\item[$1_{\mathrm{pr}}$] indicator function for predicate $\mathrm{pr}$

  %\item[$\theta,p,q$] initial probabilities of success, FP, and FN in the stationary Markov model
  \item[${\widetilde \theta}_j,{\widetilde p}_j,{\widetilde q}_j$] transition probabilities to success, FP, and FN from state $j$, $j=1,2,3$
  \item[$\eta$] confidence/weight assigned to the i.i.d. assumption in the Markov-model sensitivity analysis
\end{description}
}

%\IEEEraisesectionheading{
\section{Introduction}
%}
\label{sec_intro}

% \section{Introduction}
% \label{sec_intro}
\IEEEPARstart{S}{oftware} is required to be very reliable if its purpose is to take actions towards ensuring safety: sufficiently reliable software will take correct actions when needed, sufficiently often. When assessing such software, an assessor can use statistical inference to draw conclusions about whether the software is sufficiently reliable. The assessor does this by applying a statistical model (of the software's failure behavior) to data gathered from the software's operational past. This model must be appropriate---in particular, the model's fidelity must be compatible with the level of detail in the operational data. For example, for on-demand systems, the data can be an enumeration of the outcomes of the software's past actions, indicating whether an action was correctly taken or the software failed. An appropriate statistical model to use with such data might be a parametric family of generalized \emph{Bernoulli random processes}, if the software's failure behavior can be argued to be adequately modeled by such processes. 

Ideally, data from past software operation is gathered for a \emph{particular} assessment that uses a \emph{particular} statistical model. Realistically, this is often not the case. Gathering the data at the required fidelity might be too expensive (e.g. gathering the data requires an infeasible amount of human judgment, time, or is unaffordable in monetary terms), or the gathered data was intended for a different statistical model, or for a different purpose altogether, or the desired fidelity is unobtainable.

Consequently, the data granularity can vary significantly, depending on the circumstances surrounding the data collection. The data might be ``coarse--grained''; for instance, it only indicates when successes/failures of the software occurred, but not the nature or type of each success/failure. If the software were, say, a binary classifier that should raise an alarm when unsafe driving conditions occur while traveling in an \emph{autonomous vehicle} (AV), it would be important to know \emph{what} failure mode occurred in addition to knowing that a failure \emph{has} occurred; for instance, whether the failure was a \emph{false positive} (FP) or \emph{false negative} (FN) classification of the driving event. By omitting such details, coarse--grained reliability data can leave an assessor with some uncertainty about the software's \emph{past} failure behavior; uncertainty that should be explicitly incorporated into any reliability claims (based on this data) about the software's \emph{future} failure behavior. An assessor's uncertainty about the software's operational past should temper any claims of the software being reliable.  

What is the impact---of using operational data with missing detail and the uncertainty this brings---in reliability assessments? When is there little impact: reliability claims are as trustworthy as claims made \emph{had} the data been more detailed? When is the impact significant: claims are dangerously optimistic, leading to misplaced trust in unreliable software?

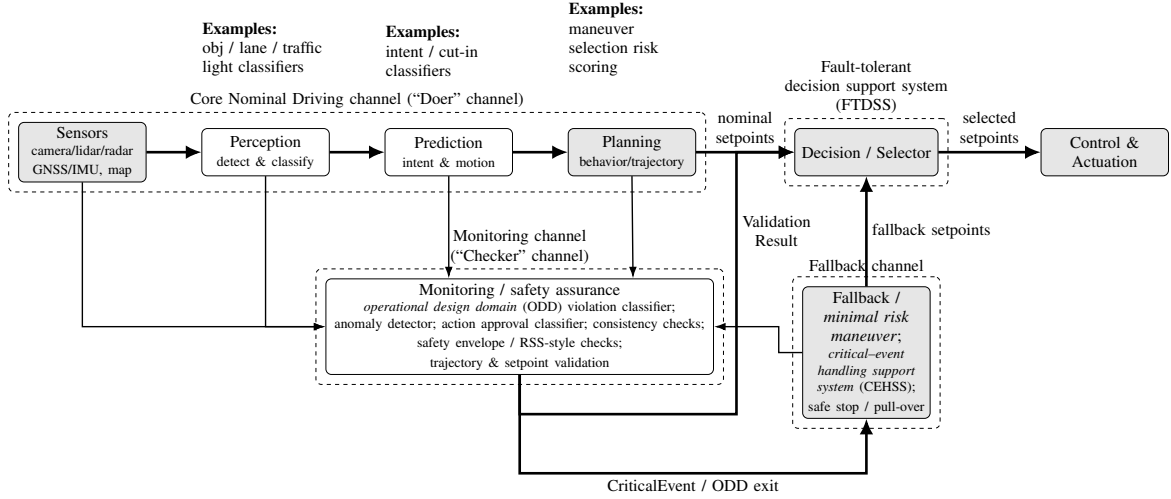
\begin{figure*}[t!]
\centering

\begin{tikzpicture}[
  font=\footnotesize,
  block/.style={
    draw,
    rounded corners,
    align=center,
    minimum height=8.5mm,
    text width=18mm,
    inner sep=0.6mm
  },
  decision/.style={
    draw,
    rounded corners,
    align=center,
    minimum height=8.5mm,
    text width=22mm,
    inner sep=0.6mm
  },
  binary/.style={
    fill=gray!20
  },
  wideblock/.style={
    draw,
    rounded corners,
    align=center,
    minimum height=11mm,
    text width=76mm,
    inner sep=0.8mm
  },
  note/.style={
    align=center,
    font=\scriptsize,
    text width=21mm,
    inner sep=0pt
  },
  arrow/.style={
    -{Latex[length=1.8mm,width=1.3mm]},
    line width=0.75pt
  },
  feed/.style={
    -{Latex[length=1.25mm,width=0.9mm]},
    line width=0.55pt
  },
  box/.style={
    draw,
    dashed,
    rounded corners,
    line width=0.7pt
  },
  grouptext/.style={
    font=\scriptsize\bfseries,
    align=center,
    inner sep=0pt
  },
  arrowlabel/.style={
    font=\scriptsize,
    align=center,
    fill=white,
    inner sep=0.3mm
  }
]

% Layers
\pgfdeclarelayer{bg}
\pgfsetlayers{bg,main}

%--------------------------------------------------------------------
% Nominal driving channel
%--------------------------------------------------------------------
\node[block,binary] (sensors)
  {Sensors\\
   \scriptsize camera/lidar/radar\\
   GNSS/IMU, map};

\node[block,right=2.2mm of sensors] (perception)
  {Perception\\
   \scriptsize detect \& classify};

\node[block,right=2.2mm of perception] (prediction)
  {Prediction\\
   \scriptsize intent \& motion};

\node[block,binary,right=2.2mm of prediction] (planning)
  {Planning\\
   \scriptsize behavior/trajectory};

%--------------------------------------------------------------------
% Fault-tolerant decision / selector
%--------------------------------------------------------------------
\node[decision,binary,right=13mm of planning] (ftdss)
  {Decision / Selector};

%--------------------------------------------------------------------
% Control & actuation
%--------------------------------------------------------------------
\node[block,binary,right=13mm of ftdss] (control)
  {Control \&\\
   Actuation};

%--------------------------------------------------------------------
% Classifier callouts
%
% Raised slightly to provide clearer separation from the nominal
% driving-channel label.
%--------------------------------------------------------------------
\coordinate (notey) at ([yshift=6.5mm]perception.north);

% \node[note,anchor=south] (noteper)
%   at (perception.center |- notey)
%   {\textbf{Examples:}\\
%    obj / lane / traffic light classifiers};

% \node[note,anchor=south] (notepred)
%   at (prediction.center |- notey)
%   {\textbf{Examples:}\\
%    intent / cut-in classifiers};

% \node[note,anchor=south] (noteplan)
%   at (planning.center |- notey)
%   {\textbf{Examples:}\\
%    maneuver selection risk scoring};

%--------------------------------------------------------------------
% Monitoring / safety assurance channel
%--------------------------------------------------------------------
\coordinate (coremid) at
  ($(perception)!0.5!(prediction)$);

\node[
  wideblock,
  below=13mm of coremid
] (mss)
  {Monitoring / safety assurance\\
   \scriptsize
   \emph{operational design domain} (ODD) violation classifier;\\
   anomaly detector; action approval classifier; consistency checks;\\
   safety envelope / RSS-style checks; trajectory \& setpoint validation};

%--------------------------------------------------------------------
% Critical event handling / fallback channel
%--------------------------------------------------------------------
\node[
  block,
  binary,
  text width=27mm
] (cehss)
  at ([xshift=9.5mm]ftdss |- mss)
  {Fallback / \emph{minimal risk maneuver};\\
   \scriptsize
   \emph{critical event--handling subsystem} (CEHSS);\\
   safe stop / pull-over};

%--------------------------------------------------------------------
% Grouping boxes
%--------------------------------------------------------------------
\begin{pgfonlayer}{bg}

  \node[
    box,
    inner xsep=1.7mm,
    inner ysep=2mm,
    fit=(sensors)(perception)(prediction)(planning)
  ] (ccdssbox) {};

  \node[
    box,
    inner xsep=2mm,
    inner ysep=2mm,
    fit=(mss)
  ] (mssbox) {};

  \node[
    box,
    inner xsep=2.3mm,
    inner ysep=2.3mm,
    fit=(cehss)
  ] (cehssbox) {};

  \node[
    box,
    inner xsep=2.3mm,
    inner ysep=2.3mm,
    fit=(ftdss)
  ] (ftdssbox) {};

\end{pgfonlayer}

%--------------------------------------------------------------------
% Nominal flow
%--------------------------------------------------------------------
\draw[arrow]
  (sensors.east) -- (perception.west);

\draw[arrow]
  (perception.east) -- (prediction.west);

\draw[arrow]
  (prediction.east) -- (planning.west);

\draw[arrow]
  (planning.east) --
  node[midway,above=0.7mm,arrowlabel]
    {nominal\\setpoints}
  (ftdss.west);

\draw[arrow]
  (ftdss.east) --
  node[midway,above=0.7mm,arrowlabel]
    {selected\\setpoints}
  (control.west);

%--------------------------------------------------------------------
% Fallback proposes setpoints to selector
%--------------------------------------------------------------------
\coordinate (fbtarget) at
  (cehss.north |- ftdss.south);

\draw[arrow]
  (cehss.north) --
  node[midway,right=0.9mm,arrowlabel]
    {fallback\\setpoints}
  (fbtarget);

%--------------------------------------------------------------------
% What the monitor observes
%--------------------------------------------------------------------
\coordinate (mssS) at
  (sensors.south |- mss.north);

\coordinate (mssP) at
  (perception.south |- mss.north);

\coordinate (mssR) at
  (prediction.south |- mss.north);

\coordinate (mssL) at
  (planning.south |- mss.north);

\draw[feed]
  (sensors.south) -- (mssS);

\draw[feed]
  (perception.south) -- (mssP);

\draw[feed]
  (prediction.south) -- (mssR);

\draw[feed]
  (planning.south) -- (mssL);

% Monitor observes fallback proposal.
\draw[feed]
  (cehss.west) -- (mss.east);

%--------------------------------------------------------------------
% Validation Result -> selector
%--------------------------------------------------------------------
\coordinate (valstart) at
  ([yshift=2.5mm]mss.east);

\coordinate (valtarget) at
  ([xshift=-7.5mm]ftdss.south);

\coordinate (valcorner) at
  (valtarget |- valstart);

\draw[arrow]
  (valstart) --
  node[midway,above=0.5mm,arrowlabel]
    {Validation\\Result}
  (valcorner) --
  (valtarget);

%--------------------------------------------------------------------
% Critical event / ODD exit -> fallback
%
% Horizontal arm moved slightly downward. The label is now slightly
% above that arm, giving it more separation from both the line and
% the nearby dashed grouping boundaries.
%--------------------------------------------------------------------
\coordinate (critstart) at
  ($(mss.south west)!0.78!(mss.south east)$);

\coordinate (critdown) at
  ([yshift=-4.0mm]critstart);

\coordinate (critright) at
  (cehss.south |- critdown);

\draw[arrow]
  (critstart) --
  (critdown) --
  node[midway,above=0.5mm,arrowlabel]
    {CriticalEvent / ODD exit}
  (critright) --
  (cehss.south);

%--------------------------------------------------------------------
% Group labels
%--------------------------------------------------------------------
\node[
  grouptext,
  above=0.4mm of ccdssbox.north
] (ccdsslab)
  {Core Nominal Driving channel (``Doer'' channel)};

\node[
  grouptext,
  text width=26mm,
  above=0.4mm of ftdssbox.north
] (ftdlab)
  {Fault-tolerant decision\\
   subsystem (FTDSS)};

\node[
  grouptext,
  text width=23mm,
  left=0.8mm of mssbox.west
] (msslab)
  {Monitoring channel\\
   (``Checker'' channel)};

\node[
  grouptext,
  text width=20mm,
  right=0.8mm of cehssbox.east
] (cehlab)
  {Fallback channel};

\end{tikzpicture}

\vspace{-2mm}

\caption{An AV autonomy stack schematic indicating typical stages in the stack, applications of classifiers (i.e. white boxes) with a focus on the monitoring channel, and how a channel--wise Doer/Checker/Fallback safety architecture can gate actuation.}
\label{fig:av_pipeline_classifiers}
\end{figure*}

This paper presents statistical techniques that conservatively estimate this impact. These novel extensions of \emph{conservative Bayesian inference} (CBI) support evidence--based, conservative, reliability claims. They provide a means to conservatively check the robustness of reliability claims based on insufficiently detailed data. The paper's research contributions are:
%\begin{enumerate}[wide,label={\arabic*)}]
\begin{enumerate}[label={\arabic*)}]
\item formally incorporating assessor uncertainty about failure data details into reliability assessments (see Section~\ref{sec_methodology});

%\item  sensitivity analyses that shows how assessments based on the i.i.d. assumption can be extremely optimistic. Such analyses can reveal ranges of values for a statistical model's parameters over which optimism occurs (see section \ref{sec_results});
   
\item novel CBI techniques that account for this uncertainty and quantify the impact this has on the conservative reliability claims an assessor can justify (see Sections~\ref{sec_methodology}--\ref{sec_discussion});
    
\item  four theorems, with and without failure data uncertainty, that give conservative posterior confidence bounds on the probability of software failure  (see Section~\ref{sec_results}).

\item demonstrating how attempts at conservatively using coarse operational data in assessments can lead to optimistic claims (see Sections~\ref{sec_applications}, \ref{sec_discussion});

\item guidance for statistical AV--safety assurance under operational--data constraints (see Sections~\ref{sec_practicalscenario}, \ref{sec_methodology}, \ref{sec_discussion}).

\end{enumerate}
The paper's outline is: Section~\ref{sec_relwork} highlights related work, Section~\ref{sec_practicalscenario} describes the reference AV scenario, and Section~\ref{sec_methodology} details the CBI approach. Section~\ref{sec_results} presents worst-case confidence bounds on classifier reliability under different degrees of uncertainty about the failure-types exhibited by the classifier during operation. The confidence bounds are applied in Section~\ref{sec_applications}. Section~\ref{sec_discussion} discusses practical considerations and implications, with concluding remarks in Section~\ref{sec_conc}.

\section{Related Work}
\label{sec_relwork}  
\subsection{AV Safety Standards}
The functional--safety baseline for production road vehicles is the ISO~26262 series (second edition, 2018). The standard targets hazards due to malfunctioning safety--related electrical/electronic subsystems and prescribes a risk--based safety lifecycle from conception through validation \cite{ISO26262ISONews,Debouk2019ISO26262Overview}. Crucially, for software--intensive AV stacks (e.g. Fig.~\ref{fig:av_pipeline_classifiers}), ISO~26262 goes beyond system--level process guidance: it explicitly requires artifacts and evidence at the \emph{software--component level}, including specification of software safety requirements, software unit verification, software integration and verification, and testing of embedded software \cite{ISO2626262018}. It is reinforced by supporting process requirements for verification planning, change/configuration management, confidence in software tools and pre-existing software components (including ``proven--in--use'' style arguments); these emphasize explicit assessment/justification of safety--relevant software elements \cite{ISO26262-8-2018}. ISO~26262 uses \emph{hazard analysis and risk assessment} (HARA) to rate hazardous events by severity, exposure, controllability, and to assign an \emph{Automotive Safety Integrity Level} (ASIL) \cite{Debouk2019ISO26262Overview}. The ASIL rating sets the level of rigor and independence for development and assurance activities \cite{ChonnadLitovtchenkoSargsyan2023ConfirmationMeasuresISO26262}.

ISO~26262 is fault--oriented. Complementary standards address AV--relevant risks not directly attributable to component faults. ISO~21448 (SOTIF) provides an argument framework for avoiding unreasonable risk due to functional insufficiencies, and explicitly calls for design, verification and validation measures (and operational--phase activities) for automated--driving and emergency--intervention functions that depend on complex sensors and processing algorithms \cite{ISO21448-2022}. ISO/PAS~8800 provides a framework/guidance for building convincing safety assurance claims for \emph{artificial-intelligence} (AI)-based safety-related systems within the vehicle. ISO/PAS~8800 tailors and extends ISO~26262 and ISO~21448 for AI elements \cite{ISOPAS8800-2024}. Finally, ANSI/UL~4600 offers a deployment--facing structure for scrutinizing evidence used to justify safety claims for autonomy--relevant subsystems \cite{UL4600-Ed2-2022}.

While not a standard, \emph{Assurance~2.0} argues that autonomous systems assurance must be more rigorous, evidence--centered, continuous/incremental, with  explicit treatment of defeaters and probabilistic treatments of evidence \cite{BloomfieldRushby2020A2Manifesto,BloomfieldRushby2022A2Confidence,ChenEtAl2025A2SDV}.

\subsection{(Conservative) Bayesian Methods for Assessments}
\label{subsec_cbiliterature}
CBI supports conservative assessments of software reliability using operational data and evidence-based beliefs about the software's propensity to fail/succeed \cite{BishopEtAl2011TSE}. This includes support for claims about fault--freeness/perfection and very small failure probabilities, using failure--free operational evidence and experience from previous similar systems \cite{StriginiPovyakalo2013SAFECOMP,littlewood_reliability_2020,salako_conservative_2021}. Preliminary work applied CBI to binary classifier assessment \cite{Salako_QEST_2020}, and to highlighting circumstances where demonstrations of desirable levels of reliability from AV operational testing are infeasible \cite{zhao_assessing_2019,zhao_assessing_2020}. CBI has been used to estimate the impact of erroneous modeling assumptions on assessments, such as incorrectly assuming statistically independent outcomes during software testing \cite{SalakoZhao2023TSE,SalakoZhao2024QRE}. Bishop \emph{et al.} \cite{BishopPovyakaloStrigini2022Bootstrapping} formalize ``confidence bootstrapping'' for AV assessment: staged deployment policies are coupled to CBI-derived lower bounds on failure--free operation over a finite future ``confidence horizon'', so that risk is contained early while evidence from safe operation tightens conservative predictions over time. To account for evolving AVs, Aghazadeh-Chakherlou \emph{et al.} \cite{ChakherlouSalakoStrigini2022ArguingSafety} extend previous CBI results for arguing post--change safety from pre--change safe operation under improvement assumptions, while \cite{AghazadehChakherlouStrigini2024PreChange} incorporates pre--change operational evidence and prior confidence in fault-freeness, showing how earlier operation can improve dependability claims when post--change data are scarce. %\cite{BishopPovyakaloStrigini2025Doubt} argues that extreme safety claims should explicitly account for inevitable epistemic doubt in assurance arguments, and show how conservative bounds during early operation can be strengthened by ``fall--back'' arguments. 
CBI aligns with robust Bayesian and imprecise-probability frameworks for inference that remains valid under uncertainty about the inputs to inference \cite{Berger1990RobustBayesSensitivity,BergerBerliner1986AoS,Walley1991ImpreciseProbabilities,moreno1991robust,berger1994_robustBayesianOverview,SalakoMuhammad2025}. %
%In \cite{} CBI is contextualized within these wider frameworks.

%CBI can also quantify the consequences of \emph{statistical model misspecification} induced by uncertainty about missing, important, operational-data details. When there is no such uncertainty, under appropriate regularity conditions, standard Bayesian posteriors concentrate around the \emph{Kullback--Leibler} (KL) projection of the data-generating distribution when the likelihood, or the prior, are misspecified \cite{kleijn2006misspecification}. In the present paper, the dominant uncertainty is about the prior and the data: low-fidelity logging of AV operational-data effectively ``\emph{coarsens}'' the observations (e.g. collapsing FP/FN types) and the likelihood. This aligns with the ``coarse data'' / ``coarsening at random'' literature, which emphasizes how inferential strength can change drastically depending on what is observed and the coarsening mechanism \cite{heitjan1991ignorability,gill1997coarsening}. From this perspective, CBI plays an analogous role to \emph{set-identified} inference that bounds posterior confidence over so-called \emph{credal sets}---sets of priors consistent with the available evidence \cite{Cozman2000,CozmanRochaCampos2002ComputingSets}.
\subsection{Some AV Safety Assessment Approaches}
The infeasibility of demonstrating AV reliability from road testing alone has been argued using classical frequentist statistics \cite{KalraPaddock2016}. Multivariate Bayesian inference (that explicitly models uncertainty about AV operating modes) has been proposed to assess AV safety using fleet operational data \cite{Popov2025DyAVSA}. Scenario-based statistical testing involves sampling representative and/or challenging traffic scenarios that AVs may encounter \cite{RiedmaierEtAl2020ScenarioSurvey}; however, Zhao et al. argue for such testing to have a stronger statistical foundation for stopping rules, residual-risk estimation, and claims about simulation fidelity \cite{ZhaoEtAl2025StatFoundation}. Importance sampling and cross-entropy methods accelerate exposure to rare, safety-critical scenarios \cite{ZhaoEtAl2017LaneChangeIS} and have been applied in end-to-end AV testing \cite{OKellyEtAl2018RareEvent}. These approaches typically require fully specified priors, high-fidelity typed data, or validated stochastic/simulation models. The present paper addresses a complementary problem: when operational logs are too coarse, what conservative claims remain justified?

\subsection{Some Binary Classifier Assessment Approaches}
\label{subsec_roc}
Binary classifier assessment estimates Bernoulli error rates (e.g. FN and FP probabilities) from operational data. Empirical evaluation typically involves summary statistics of a classifier's failure behavior; statistics including accuracy, sensitivity (\emph{true positive rate} (TPR)/recall), specificity (\emph{true negative rate} (TNR)), precision (positive predictive value), and $F_1$ score (harmonic mean of precision and recall) \cite{AltmanBland1994,SokolovaLapalme2009,DavisGoadrich2006}. Performance is frequently reported across thresholds via \emph{receiver operating characteristic} (ROC) curves (TPR \emph{vs.} \emph{false positive rate} FPR) \cite{Fawcett2006} and the associated \emph{area} \emph{under} \emph{the} \emph{curve} (AUC) \cite{HanleyMcNeil1982}. Confusion matrices reveal trade-offs between failure rates and the types/frequency of failures/successes at different decision thresholds. However, ROC/AUC can be misleading under strong class imbalance, motivating complementary \emph{precision--recall} (PR) curves and PR-focused summaries \cite{SaitoRehmsmeier2015,DavisGoadrich2006}. 

Bayesian ROC methods propagate observational and distributional uncertainty into ROC/AUC trade-offs. A substantial literature treats ROC curves as \emph{random} objects induced by prior distributions over class-conditional score distributions. %; this yields posterior uncertainty bands for ROC/AUC analyses. 
Also, semi/non-parametric Bayesian constructions can be more robust than binormal-style assumptions while still supporting principled threshold selection and credible intervals \cite{Erkanli2006,Branscum2008,Gu2008}. Bayesian ROC also helps with ``ground truth'' uncertainty \cite{Choi2006}. %Hierarchical Bayesian extensions further address dependence/heterogeneity that arise with clustered observations or multi-reader study designs \cite{OMalleyZou2006,JohnsonJohnson2006}, and modern models allow ROC performance to vary with covariates rather than assuming a single global curve \cite{RodriguezMartinez2014}. %There are also hierarchical ROC frameworks \cite{RutterGatsonis2001}.

\section{Reference AV Scenario}
\label{sec_practicalscenario}
%\subsection{From AV Driving Episodes to Operational Evidence}
Fig.~\ref{fig:av_pipeline_classifiers} exemplifies a recursive AV driving pipeline. During a journey, the AV repeatedly senses, localizes, perceives, predicts, plans, checks, selects, controls, actuates, and then observes the consequences of its action. This perception--planning--control decomposition follows traditional AV software tasks \cite{Pendleton2017Machines}. The planning and decision-making context, including safety and verification concerns, is consistent with \cite{Schwarting2018AR}. The Doer/Checker/Fallback architecture follows \cite{TheAutonomousWG2023SafetyArch}. 
%The explicit monitoring/checking channel reflects autonomy-stack schemes in which planning and control are supervised by verification or safety-assurance mechanisms \cite{Schwarting2018AR}. 
A typical \emph{machine}-\emph{learning} (ML)-based nominal channel consumes camera, LiDAR, radar, \emph{global} \emph{navigation} \emph{satellite system} (GNSS), \emph{inertial} \emph{measurement} \emph{unit} (IMU), and map inputs; constructs fused scene representations; detects and classifies objects, lanes, drivable space, and traffic lights; predicts agent intent and motion; generates behavior or trajectory candidates; and produces nominal setpoints. In parallel, the checker channel evaluates the nominal channel's intermediate and final outputs. In the Doer/Checker/Fallback structure in Fig.~\ref{fig:av_pipeline_classifiers}, a monitoring subsystem and fault-tolerant selector can approve nominal setpoints, reject them, or switch to a minimal-risk fallback channel \cite{TheAutonomousWG2023SafetyArch}. 
%\subsection{Binary Classification in AV Safety}
The safety-relevant binary decisions can be assessed as binary classifier outputs. %To make explicit how AV operational data become the observations used in Section~\ref{sec_methodology}, consider the following compact example for a checker-channel safety monitor $C_{\mathrm{mon}}$ that either \emph{allows} a nominal proposal or \emph{intervenes}. 
To illustrate how AV operational data become the observations used in Section~\ref{sec_methodology}, consider a checker-channel safety monitor $C_{\mathrm{mon}}$ that either \emph{allows} or \emph{intervenes} on a nominal proposal.

%Such monitors have been modelled as binary safe/unsafe classifiers \cite{TerrosiStriginiBondavalli2022ImpactMLSafetyMonitors}.

\emph{1) Logging operational data.} On invocation $i$, $C_{\mathrm{mon}}$ receives demand input $X_i$,
comprising the AV state/history required by the monitor---e.g.,
ego state, relevant object tracks and the nominal trajectory or
setpoint---and returns
$C_{\mathrm{mon}}(X_i)\in\{\texttt{allow},\texttt{intervene}\}$.
Here, each classifier invocation constitutes a classification, not each camera,
LiDAR or CAN sample. Runtime-monitor evaluations make this distinction: Farid \emph{et al.} execute their detector at each planning
cycle and evaluate it on prediction horizons labelled for
task-relevant prediction errors independently of the detection
method \cite{FaridEtAl2023TaskRelevantFailure}.

In principle, every invocation can be logged with the classifier/configuration version, the
synchronized input trace needed for subsequent adjudication, and
ODD/mode tags. While such logging may not always be practicable, thereby restricting data fidelity, there is precedent: e.g. in the NHTSA Safety Pilot Model Deployment, a large-scale field
evaluation of connected-vehicle safety applications,
Nodine \emph{et al.} \cite{NodineEtAl2015SafetyPilot} continuously
recorded CAN/V2V/external-sensor data at $10$\,Hz with synchronized
video; analysts later reviewed these records and annotated contextual
attributes such as road geometry, manoeuvre, lighting, and weather.  Let $OS$ denote the chosen operational slice, e.g.,
\emph{urban/day/dry, signalized intersection, ego speed 10--15\,m/s,
software v5.2}.  Let $OS'$ denote a nearby slice that differs only in speed band.  Successive demands $X_i$ and outputs $C_{\rm mon}(X_i)$ may
occur in $OS,OS',OS,\ldots$ as conditions vary.%; the assessment retains only those in $OS$, avoiding pooled FP/FN probabilities across different slices.

\emph{2) Adjudicate ground truth from future observations.}
Let $G_i\in\{\texttt{allow},\texttt{intervene}\}$ be the desired
output under a pre-specified ground-truth rule, determined independently
of $C_{\mathrm{mon}}$'s own score; e.g. $G_i=\texttt{intervene}$ iff continuing with the nominal proposal would violate the pre-specified safety condition; otherwise $G_i=\texttt{allow}$. If neither the current decision nor a subsequent intervention
alters the future needed to apply that rule, the subsequently
observed telemetry determines $G_i$. Farid
\emph{et al.} evaluate a runtime monitor of trajectory
predictions by comparing predicted motions of surrounding road users
with their subsequently observed trajectories, independently of the monitor's output \cite{FaridEtAl2023TaskRelevantFailure}.  Similarly, the Safety Pilot used synchronized video and telemetry
within a classification framework distinguishing valid alerts,
false alerts, missed alerts, and valid rejections \cite{NodineEtAl2015SafetyPilot}.

\emph{3) Adjudicate ground truth counterfactually.}
If $C_{\mathrm{mon}}$ outputs \texttt{intervene}, the intervention may
change the future needed to determine $G_i$; a safe outcome after
braking does not necessarily show braking was necessary.  The unobserved \texttt{allow} branch must instead be reconstructed
from the logged pre-intervention data.  Terrosi \emph{et al.} record simulated controller runs and replay
them with the safety monitor enabled, because its intervention can
alter or remove an accident sequence \cite{TerrosiStriginiBondavalli2022ImpactMLSafetyMonitors}. Using public-road AV data, Schwall \emph{et al.} construct counterfactual AV trajectories that would have occurred. %, initializing the simulation from the recorded pre-disengagement vehicle state, sensor observations and software version running at the time. 
Possibly altered responses from other road users are also modelled
\cite{SchwallEtAl2020WaymoPublicRoad}.

%Thus, each invocation receives one reference label $G_i$, obtained either factually (F) or, where intervention obscures the relevant future, counterfactually (CF).  
Table~\ref{tab:worked_log} illustrates
the resulting records. Taking $\texttt{intervene}$ as the positive output, comparison of $C_{\rm mon}(X_i)$ with $G_i$ gives TP/TN/FP/FN. The telemetry triples are a compact extract of $X_i$:
$v$ is ego-vehicle speed ($\texttt{m/s}$), $d$ is the distance to the relevant
tracked object ($\texttt{m}$), and $\Delta v$ is the relative approach speed
($\texttt{m/s}$, positive when the distance is decreasing); A/I denote
\texttt{allow}/\texttt{intervene}. Each invocation's ground truth label is determined factually (F) or counterfactually (CF).
\begin{table}[t]
\centering
\caption{Synthetic excerpt from an adjudicated invocation log.}
\label{tab:worked_log}
%\scriptsize
\setlength{\tabcolsep}{8.1pt}
\begin{tabular}{@{}rccccc@{}}
\hline
$i$ & slice & $(v,d,\Delta v)$ & $C_{\mathrm{mon}}(X_i)$ & $G_i$/basis & result\\
\hline
417 & $OS$ & $(14.9,24.6,7.4)$ & A & A/F & TN\\ 418 & $OS'$ & $(15.1,23.9,7.6)$ & I & I/CF & TP\\ 419 & $OS$ & $(14.9,23.1,7.5)$ & I & I/CF & TP\\ 420 & $OS$ & $(14.8,22.4,7.3)$ & I & A/CF & FP\\ 421 & $OS$ & $(14.6,21.6,7.1)$ & A & A/F & TN\\ 422 & $OS$ & $(14.5,20.9,7.3)$ & A & I/F & FN\\
\hline
\end{tabular}
\end{table}

\emph{4) Restrict to chosen operational slice $OS$.}
Suppose the complete log contains $n=50{,}000$ invocations in slice
$OS$, of which ground-truth adjudication gives $k_1=2$ FPs, $k_2=2$ FNs and
$49{,}996$ correct classifications.  
% Since our model deliberately
% does not distinguish the two kinds of correct classification, TP
% and TN collapse to \emph{success}, giving
% \[
%         n=50{,}000,\qquad k_1=2,\qquad k_2=2 .
% \]
These are the operational observations supplied to the categorical
model in Section~\ref{sec_methodology}. %; the subsequent CBI calculation therefore starts from an auditable chain $X_i\rightarrow C_{\mathrm{mon}}(X_i)\rightarrow G_i \rightarrow\{\mathrm{success},\mathrm{FP},\mathrm{FN}\}$ rather than from pre-aggregated failure counts.

The foregoing assumes that ground-truth adjudication is
sufficiently credible.  While counterfactual simulation can estimate what would plausibly have happened, it is model-dependent and may not be definitive \cite{SchwallEtAl2020WaymoPublicRoad}.  Uncertainty in ground-truth labels is a challenge not modelled in our work.

\section{Methodology}
\label{sec_methodology}
The following statistical model is used in Section~\ref{sec_results} to show the impact of data-fidelity on assessments. Unlike software reliability models with failure events in continuous execution/calendar time (e.g. \cite{GosevaPopstojanovaTrivedi2000,GoelOkumoto1979}), the modeling framework adopted here is discrete and ``demand''--indexed (\emph{cf.} \cite{zhao_assessing_2019,littlewood_validation_1993}). This distinction aligns naturally with system type: continuous-time models are well suited to continuously operating systems (e.g. transaction-processing or telecommunications software), whereas safety-critical decision functions are modeled as on-demand systems with reliability defined per actuation/decision.

\subsection{A Bayesian Model of Binary Classification}
\label{subsec_binaryclass_statsmodel}
A binary classifier receives a sequence of \emph{s-independent and identically distributed} (i.i.d.) classification samples. To be consistent with the i.i.d. assumption, each observation is
one classifier invocation (i.e., one demand), drawn from a stable
operational profile and a fixed operational slice.  Successive
invocations may nevertheless be temporally or logically coupled, so the i.i.d. assumption is only approximate, at best. The impact of assuming i.i.d. on assessment claims can be evaluated using CBI sensitivity analyses---see Section~\ref{sec_discussion}.
%To be consistent with the i.i.d. assumption, data collection should sample event-level demands---e.g. one lane-change-gap decision, one intersection-approach decision, or one trajectory-approval decision per triggered episode---from a stable operational profile. AV operational modes matter because mixing lane keeping, intersections, merging, degraded perception, and fallback episodes may mix different FP/FN probabilities. Even with careful sampling, the approximation is imperfect: closed-loop interventions alter subsequent traffic states, and real-world logs are temporally and logically coupled\footnote{This i.i.d. assumption should be understood to hold approximately, at best. The impact of this assumption on a given assessment's claims can be evaluated using CBI/robust inference techniques---see Section~\ref{sec_discussion}.}. 

The classifier labels each sample as either belonging to class ``1'' or class ``0''. There are 4 possible outcomes on each sample: the classifier can succeed in two ways (i.e. correctly classifying the sample as belonging to class ``$1$'' or ``$0$'') or fail in two ways (i.e. an FP failure, where a ``$0$'' sample is misclassified as ``$1$'', or an FN failure, where a ``$1$'' sample is misclassified as ``$0$''). Ignoring differences in success types\footnote{This analysis is failure-centric, which is understandable since the concern is safety. The techniques can be extended to cater for multiple success types.}, only 3 outcomes are possible---either success or one of two failure types. Which of these 3 outcomes will occur is not known before the classifier receives the sample. So, there are 3 probabilities an assessor would like to know: the probability $\Theta$ that the software succeeds, the probability $P$ of an FP failure, and the probability $Q$ of an FN failure. By definition, $\Theta+P+Q=1$, so the assessor's uncertainty reduces to not knowing the true values of only $P$, $Q$. The quantity ``$P+Q$'' is the unknown \emph{probability of failure per classification} (\emph{pfc}). All of this defines an unknown \emph{categorical random process}---parameterized by $(P,Q)$---that models the occurrence of the classifier's successes/failures. The probabilities $P$, $Q$, $\Theta$ are unconditional; \emph{cf.} conditional FP/FN rates in ROC analyses.

Prior to observing the classifier in operation, the assessor forms beliefs---represented as some probability distribution $\mathbb P$---about how likely various values of $(P, Q)$ are. These beliefs are based on reliability evidence gathered before observing the classifier in operation (see section~\ref{subsec_practicalcontext}). Then, upon observing the classifier make $k_1$ FPs, $k_2$ FNs, and $n-k_1-k_2$ correct classifications on $n$ classification samples\footnote{The assessor notes each  outcome on each sample in the sequence.}, an assessor updates their (prior) beliefs via Bayesian inference. As a consequence, the operational evidence can sway the assessor into being more/less skeptical about, say, the classifier having very good accuracy. That is, the assessor's updated doubt in the classifier's accuracy being good enough is
\begin{equation}
{\mathbb P}(\Theta< 1-b\mid n\mbox{ tasks, }k_1\mbox{ FPs, }k_2\mbox{ FNs})
\label{eqn_upperconfbnd_accuracy}
\end{equation}
for some sufficiently small $b\in[0,1]$ such that $1-b$ is the minimum accuracy required of the classifier.

Priors can often be difficult to specify or justify fully; e.g. for the assessment scenarios presented in this paper, where the prior $\mathbb P$ is a 2-dimensional joint distribution over feasible $(P,Q)$ values. A pragmatic alternative is the assessor specifies only those aspects of $\mathbb P$ that available evidence justifies. This often means that only a \textit{partial specification} of the prior is provided; e.g. specifying 90\% confidence in the classifier's accuracy being greater than $0.999$. Formally, a partial specification of the prior defines a set $\mathcal D$ of all prior distributions that, each, satisfy the assessor's beliefs. Restricting $\mathcal D$ to priors with well-defined posterior probabilities, one can determine the worst support the priors in $\mathcal D$ can give for a reliability claim, by finding the worst-value these priors can give for a posterior probability of interest. This is the CBI approach. For example, rather than computing \eqref{eqn_upperconfbnd_accuracy} using a specific prior $\mathbb P$, a skeptical assessor might want to know the most doubt they can have, after seeing evidence of reliable operation. That is, the assessor seeks a prior that solves the following:
%For example, \eqref{eqn_upperconfbnd_accuracy} quantifies an assessor's doubt---in the classifier's accuracy being good enough---after observing the classifier in operation. However, a skeptical assessor might want to know the most doubt they can have, after seeing evidence of reliable operation. That is, rather than computing \eqref{eqn_upperconfbnd_accuracy} using a specific prior $\mathbb P$, the assessor seeks a prior that solves the following: 
\begin{align}
\label{eqn_abstract_CBI_problem}
    \sup\limits_{\mathcal D}{\mathbb P}(\Theta< 1-b\mid n\mbox{ tasks, }k_1\mbox{ FPs, }k_2\mbox{ FNs}).
 \end{align}
The solution to \eqref{eqn_abstract_CBI_problem} is conservative---\emph{the smallest amount of doubt an assessor can have (based on the operational evidence) that is at least as large as any doubt the assessor can justify}! A prior that gives this conservative posterior value is a ``worst-case'' prior. Under the categorical process, \eqref{eqn_abstract_CBI_problem} is
 \begin{align}
 \label{eqn_lessabstract_CBI_problem}
    &\sup\limits_{\mathcal D}\dfrac{\mathbb E[L(n,k_1,k_2;P,Q)\mathbf{1}_{P+Q> b}]}{\mathbb E[L(n,k_1,k_2;P,Q)]}
 \end{align}
 where $L(n,k_1,k_2;p,q):=p^{k_1}q^{k_2}(1-p-q)^{n-k{1}-k_{2}}$ is the likelihood and $\mathbf{1}_{\mathtt{pr}}$ is an indicator function---it equals $1$ when an observation satisfies predicate $\mathtt{pr}$, and is $0$ otherwise.

 If the assessor is told $k$ failures occurred, but not which of these are FPs or FNs, the categorical process becomes a \emph{Bernoulli} process and \eqref{eqn_abstract_CBI_problem} is replaced by
 \begin{align}
 \label{eqn_Bernoulli_CBI_problem}
    &\sup\limits_{\mathcal D}{\mathbb P}(\Theta< 1-b\mid n\mbox{ tasks, }k\mbox{ failures }) \nonumber\allowdisplaybreaks \\
    {}={}&\sup\limits_{\mathcal D}\dfrac{\mathbb E[L(n,k;P,Q)\mathbf{1}_{P+Q> b}]}{\mathbb E[L(n,k;P,Q)]}\allowdisplaybreaks
 \end{align}
 with likelihood $L(n,k;p,q):=(p+q)^k(1-p-q)^{n-k}$.

We focus on conservatively bounding \emph{pfc} (equivalently accuracy). The impact of data fidelity is most easily appreciated by focusing on \emph{accuracy}, but the analyses presented in this paper can be extended to incorporate multiple reliability--related metrics, e.g. \emph{sensitivity} and \emph{specificity}; see Section \ref{sec_discussion}.

\subsection{Prior Beliefs about Classifier Reliability}
\label{subsec_priorbeliefs}
What are some beliefs our assessor might justify? The assessor might express various degrees of confidence related to how reliable the classifier is expected to be. In the assessor's informed opinion, the classifier's development process is expected to produce very reliable classifiers. We can state these beliefs more formally as \emph{prior knowledge} (\textbf{PK}s): for $a, b_1, l, l_1, l_2 \in [0,1]$,
\begin{constraint}
\label{PK1}
The classifier is imperfect:
$\mathbb P(P+Q\geqslant l)=1$, where the pfc, $P+Q$, cannot be smaller than $l$.% given the limits of what is achievable with current software engineering methods, available training data and hardware.
\end{constraint}
%\begin{constraint}
%\label{PK2}
%Being very confident that FP errors are very unlikely: $\mathbb P(P\geqslant b_1)=a_1 $, for $a_1, b_1$ close to $0$;
%\end{constraint}
%\begin{constraint}
%\label{PK3}
% Being very confident that FN errors are very unlikely: $\mathbb P(Q\geqslant b_2)=a_2 $, for $a_2, b_2$ close to $0$;
%\end{constraint}
\begin{constraint}
\label{PK4}
 Being very confident that the classifier's accuracy is better than required: $\mathbb P(\Theta \geqslant b_1)=a$, for $a, b_1$ close to 1;
\end{constraint}
\begin{constraint}
\label{PK5}
A more restrictive version of \textbf{PK}\ref{PK1}:
$\mathbb P(P\geqslant l_1, Q\geqslant l_2)=1$, where $P,\,Q$ cannot be smaller than $l_1,\,l_2$ in practice and the smallest possible pfc is $l_1+l_2$. So, $l_1+l_2\leqslant 1-b_1$.
\end{constraint}

\subsection{Practical Considerations/Evidence for Parameter Values}
\label{subsec_paramconstraints}
The PK parameters $a, b_1, l, l_1, l_2$ have different epistemic roles, and assessors may justify the \textbf{PK}s using various forms of evidence. Formal expert--elicitation methods help to translate an assessor's judgment---e.g. translating perceptions of the development team’s competence, the plausibility/robustness of implemented algorithms, and confidence in the engineering process/toolchain---into probability statements \cite{OHagan2019,Garthwaite2005Eliciting}. Experts may be presented with alternative failure/success scenarios and asked to express how surprised they would be (on an appropriate scale) to see such scenarios given the ancillary evidence. Recognised elicitation protocols
seek to make such judgments transparent, reproducible and less susceptible
to cognitive bias, rather than to recover an objectively existing prior
probability \cite{Garthwaite2005Eliciting,OHagan2019}; e.g. \cite{WallsQuigley2001} constructs Bayesian priors for reliability-growth testing using expert judgment, and the NUREG-1150 nuclear risk assessments used formal elicitation to obtain probability distributions for
quantities lacking sufficient experimental or analytical evidence \cite{OrtizEtAl1991}.%Elicitation can be supported by Bayesian-belief networks that take into account classifier training maturity, classification algorithm choice, and verification artifacts \cite{Fenton2007RankedNodes,SalakoStrigini2014}. %Elicitation can be carried out with a single assessor or among a group of assessors to obtain \textbf{PK}-parameter values that reflect a consensus view. Elicited \textbf{PK}s can be refined via Bayesian inference using unit/integration test evidence.

%The PK parameters $a, b_1, l, l_1, l_2$ have different epistemic roles.  In subjective Bayesianinference, expert-knowledge elicitation represents an assessor's state of knowledge by subjective probabilities; recognised elicitation protocols seek to make such judgments transparent, reproducible and less susceptible to cognitive bias, rather than to recover an objectively existing prior probability \cite{Garthwaite2005Eliciting,OHagan2019}. For example, \cite{WallsQuigley2001} constructs Bayesian Priors for reliability-growth testing using structured expert judgment, and the NUREG-1150 nuclear risk assessments used a formal elicitation process to obtain probability distributions for quantities for which experimental or analytical evidence was incomplete \cite{OrtizEtAl1991}.

In \textbf{PK}\ref{PK4}, $b_1$ is an engineering
performance target.  Safety-assurance methods for ML explicitly take
system-level safety requirements as inputs, allocate requirements to the ML
component, and derive corresponding ML safety requirements, including
performance and robustness requirements \cite{HawkinsEtAlAMLAS}.  In
contrast, probability $a$ is the assessor's prior belief that this target has been achieved.
The value of $a$ is elicited after the assessor considers pre-operational
evidence of the software’s reliability, the rigor of its development, and the adequacy of its assurance; these inform the judgment but need not determine a unique ``correct'' value of $a$
\cite{Garthwaite2005Eliciting,OHagan2019}.  Evidence can include
independent test and verification results, robustness analyses, development
and \emph{verification} \emph{and} \emph{validation} (V\&V) evidence, and experience with comparable systems.
Formal reliability and risk-assessment applications demonstrate how such evidence can be presented to experts and converted
into auditable probabilistic judgments
\cite{WallsQuigley2001,OrtizEtAl1991}. If materially different values of
$a$ remain defensible, sensitivity analysis is appropriate \cite{Berger1990RobustBayesSensitivity}; see Section~\ref{sec_discussion}.

The lower-support parameters $l, l_1, l_2$ have a different evidential interpretation.
For \textbf{PK}\ref{PK1}, evidence for $l$ can be obtained for the same
classifier, configuration and operational slice before the operational data
used for assessment.  Evidence from
sufficiently comparable classifiers can also inform the judgment of an appropriate $l$ value: e.g. large, standardised evaluations of automotive
perception systems have repeatedly measured non-zero error at the demonstrated performance frontier.  Doll\'{a}r et al. evaluated $16$
state-of-the-art pedestrian detectors across $6$ datasets and found
persistent detection errors \cite{DollarEtAl2012}; Zhang et al. subsequently investigated the gap between state-of-the-art
pedestrian detectors and a ``perfect single-frame detector'', and explicitly
analysed recurrent false-positive and missed-detection mechanisms
\cite{ZhangEtAl2016}; and CityPersons provided a 27-city driving-scene benchmark on which
strong contemporary convolutional neural network pedestrian detectors retained non-zero miss
rates \cite{ZhangEtAl2017}. 
%More recent re-evaluations continue to compare multiple modern detectors using miss rate over a range of false-positive-per-image operating points, and explicitly analyse how false-positive accounting changes apparent performance \cite{PedestrianEvaluation2024}.  
Such results can give empirical credence for an extremely small non-zero $l$,
provided similarity of task, sensing, threshold and operating conditions is
argued. More evidence may come in the form of a Bayes-error estimate: the minimum classification error
achievable for a specified data distribution, used 
to compare state-of-the-art classifiers with the theoretical
optimum \cite{TheisenEtAl2021}. An assessor's judgment that there is a small positive $l$ can be supported by an expectation that there is a residual error arising from the classifier only approximating an unknown target function, using finite training data with sensing limitations.

\textbf{PK}\ref{PK5} requires failure-mode specific evidence for $l_1$ and $l_2$. This is obtainable in principle: see the aforementioned  Farid et al., and the U.S. Safety Pilot Model Deployment studies \cite{FaridEtAl2023TaskRelevantFailure,NodineEtAl2015SafetyPilot} for
examples. 
%, evaluated several AV runtime monitors on 2,568 labelled prediction horizons, independently labelled the presence of task-relevant prediction failures, recorded the class prevalence, and reported empirical false-positive and false-negative rates \cite{FaridEtAl2023TaskRelevantFailure}.  Likewise, the U.S. Safety Pilot Model Deployment used synchronized video and vehicle data in naturalistic driving to classify safety-system outcomes as valid alerts, false alerts, missed alerts, or valid rejections \cite{NodineEtAl2015SafetyPilot}. 
%Care is required when transferring conventional classifier metrics into the parameters used here.  If $\pi_1=\Pr(Y=1)$ and $\pi_0=\Pr(Y=0)$ denote class prevalences in the assessed operational slice, then the unconditional failure probabilities are \[P=\pi_0\,\mathrm{FPR},\qquad Q=\pi_1\,\mathrm{FNR}.\]
%Consequently, 
Evidence supporting an $l_1$ choice can be assembled from a lower-floor judgment about FP behavior together with evidence about negative-class prevalence in the relevant slice; analogously, for an $l_2$ choice. %needs FN and positive-class prevalence evidence.  %An evaluation that reports only aggregate accuracy can therefore inform $l$, but does not by itself justify positive values of both $l_1$ and $l_2$.  Conversely, mode-specific studies reporting FPR, FNR and class prevalence provide exactly the ingredients needed to reason about separate unconditional floors for $P$ and $Q$.

The probability-one statements in \textbf{PK}\ref{PK1} and \textbf{PK}\ref{PK5} are support
restrictions.  Conditional on the
assessor's judgment, they exclude parameter values below the
chosen boundaries from the admissible prior class.  Benchmark, test, historical, theoretical and engineering evidence can make such an exclusion defensible, but need not deductively prove a physical limit. If the assessor is unwilling to exclude a region, the corresponding lower bound should be weakened, possibly to zero.  Sensitivity analysis should then be used to determine whether conclusions materially depend on the elicited support boundaries \cite{Berger1990RobustBayesSensitivity}.

%The parameter values must be consistent with general guidance on supporting reliability claims using statistical/operational testing \cite{StriginiLittlewood1997}. 
Typically, $b_1$ is relatively large; the development team aims to develop highly accurate classifiers. Relatedly, probability $a$ is also relatively large; the assessor is justifiably confident that the development team has produced a suitably reliable classifier. Since $1-b$ is the minimum accuracy required of the classifier, $0<1-b\leqslant b_1$. Moreover, by \textbf{PK}\ref{PK5} and the non-negativity of $l_1$ and $l_2$, $\max\{l_1,l_2\}\leqslant \min\{b, 1-b_1\}$. %Of course, \textbf{PK}\ref{PK4} implies $\mathbb P(P+Q \leqslant 1-b_1)=a$.

%The assessor also has realistic  expectations about how good the classifier \emph{can} be. For non-trivial classification problems, a classifier employing ML algorithms is likely to \emph{not} have $100\%$ accuracy, primarily due to the imprecision with which the ML algorithm approximates the unknown target classification function, and due to training the algorithm on only a sample of the data to be classified. Thus, the development team cannot produce ML--based software with $P, Q$ smaller than some very small $l_1, l_2$.

% The following two practical considerations place further restrictions on the model parameters\footnote{}. 
 
 %The classifier's developers may have a target value of $1-b_1$, for the probability of \emph{any type of} failure, that is no bigger than each target value $b_1, b_2$, for the probabilities of an FP or FN, respectively. 
 %So, formally, the events $[P\geqslant b_1]$ and $[Q \geqslant b_2]$ are both contained within the event $[P+Q\geqslant 1-b_1]$. As a consequence, the probabilities of these events must also be ordered. That is, the assessor must be more doubtful of the probability of failing a task being acceptably small, than they are doubtful of each failure-type probability being acceptably small: that is, $1-a\geqslant \max\{a_1,a_2\}$ necessarily. 
 
 % Operational evidence is typically gathered over an observation period for which the \emph{maximum likelihood estimates} of $P$, $Q$, i.e. $\frac{k_1}{n}$ and $\frac{k_2}{n}$, are larger than the target-\emph{pfc} $1-b_1$.

AV safety targets are often so small that even large practicable datasets with failures have $1/n>1-b_1$; then, any non-zero MLE $\frac{k_1}{n}$ or $\frac{k_2}{n}$ exceeds the target-\emph{pfc} $1-b_1$.

% in your preamble (or before the figure)
\newcommand{\omegaFigScale}{0.85} % change this to resize the figure

\begin{figure}[t!]
    \centering
    \begin{tikzpicture}[scale=\omegaFigScale,
                        every node/.style={scale=\omegaFigScale}] % nodes scale too
        \begin{scope}
            %vertices of sample space triangle with non-zero probability
            \coordinate (r0) at (0.05,-1.25);
            \coordinate (s0) at (0.05,3);
            \coordinate (si) at (5,-2.2);
            \coordinate (ri) at (1,-2.2);
            \fill[fill=gray!20] (r0) -- (s0) -- (si) -- (ri) -- cycle;

            %the b line
            \draw[line width=0.05cm,white] (0,0.25) -- (2.45,-2.2);

            %the 1-b3 line
            \draw[line width=0.05cm,white] (0,-1.25) -- (1,-2.2);
                
            \draw[-{Latex[length=3mm]},line width=0.02cm] (-1,-2.25) -- (5.85,-2.25); %draw horizontal axis
            \node[anchor=north,scale=1.2] (x_axis_label) at (6,-2.15) {$\pmb{p}$}; %Label horizontal axis

            %axes labels
            \node[anchor=north,scale=0.8] (N_label_2) at (2.65,-2.35) {$\pmb{1-b_1}$};
            \node[anchor=north,scale=0.8] (N_label_4) at (1.25,-2.35) {$\pmb{l}$};
            \node[anchor=north,scale=1.0] (N_label_5) at (5.2,-2.2) {$\pmb{1}$};
            \node[anchor=north,scale=0.8] (N_label_2) at (-0.5,0.55) {$\pmb{1-b_1}$};
            \node[anchor=north,scale=0.8] (N_label_2) at (-0.25,-0.9) {$\pmb{l}$};
            \node[anchor=north,scale=0.9] (N_label_5) at (-0.25,3.3) {$\pmb{1}$};

            \node[anchor=north east,scale=1] (x_axis_label) at (0.05,-2.22) {$\pmb{0}$}; %Label origin
        
            \draw[line width=0.02cm,-{Latex[length=3mm]}] (0,-3) -- (0,3.7);  %draw vertical axis
            \node[scale=1.2,left] (vertaxislabel) at (-0.05,3.6) {$\pmb{q}$}; %label vertical axis

            %preferred location labels
            \node[anchor=west,scale=0.9] (N_p15q15) at (0.75,-1.4) {$\boldsymbol S_2$}; %S2
            \node[anchor=west,scale=0.9] (N_p2q2) at (1.6,-0.3) {$\boldsymbol S_1$}; %S1
         
        \end{scope}    
    \end{tikzpicture}
    \caption{ \textbf{PK}\ref{PK1}, \textbf{PK}\ref{PK4}, partition the set $\Omega$ of feasible $(P,Q)$-values.}
    \label{fig_omega_partition_Ber_1}
\end{figure}
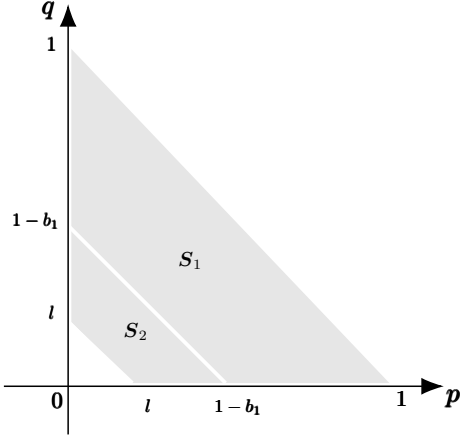

\begin{figure}[t!]
    \centering
        \begin{tikzpicture}[scale=\omegaFigScale,
                        every node/.style={scale=\omegaFigScale}] % nodes scale too
            \begin{scope}
                %\coordinate (r0) at (0.25,-2);
                %\coordinate (s0) at (0.25,3);
                %\coordinate (si) at (5,-2);
                %\fill[fill=gray!20] (r0) -- (s0) -- (si) -- cycle;
            
                %vertices of sample space triangle with non-zero probability
                \coordinate (r0) at (0.25,-2);
                \coordinate (s0) at (0.25,3);
                \coordinate (si) at (5,-2);
                %\coordinate (ri) at (1,-2);
                \fill[fill=gray!20] (r0) -- (s0) -- (si) -- cycle;

                %k/n p-stationarity line
                %\draw[line width=0.05cm,white] (0,1.4) -- (3.8,-2.5);

                %the b1 line
                %\draw[line width=0.05cm,white] (1.5,-2) -- (1.5,2.75);

                %the b2 line
                %\draw[line width=0.05cm,white] (0.25,-0.9) -- (4.25,-0.9);

                %the b line
                \draw[line width=0.05cm,white] (0,0.25) -- (2.45,-2.2);

                %the 1-b3 line
                %\draw[line width=0.05cm,white] (0,-1.25) -- (1,-2.2);
                    
                \draw[-{Latex[length=3mm]},line width=0.02cm] (-1,-2.25) -- (5.85,-2.25); %draw horizontal axis
                \node[anchor=north,scale=1.2] (x_axis_label) at (6,-2.15) {$\pmb{p}$}; %Label horizontal axiz

                %axes labels
                \node[anchor=north,scale=0.8] (N_label_2) at (2.65,-2.35) {$\pmb{1-b_1}$}; %Label horizontal axiz
                \node[anchor=north,scale=0.8] (N_label_2) at (-0.25,-1.7) {$\pmb{l_2}$}; %Label vertical axiz
                %\node[anchor=north,scale=0.8] (N_label_4) at (3.7,-2.35) {$\pmb{b}$}; %Label horizontal axiz
                \node[anchor=north,scale=1.0] (N_label_5) at (5.2,-2.2) {$\pmb{1}$}; %Label horizontal axiz
                \node[anchor=north,scale=0.8] (N_label_2) at (-0.5,0.55) {$\pmb{1-b_1}$}; %Label vertical axiz
                \node[anchor=north,scale=0.8] (N_label_2) at (0.35,-2.35) {$\pmb{l_1}$}; %Label horizontal axiz
                %\node[anchor=north,scale=0.8] (N_label_2) at (-0.25,1.6) {$\pmb{b}$}; %Label vertical axiz
                \node[anchor=north,scale=0.9] (N_label_5) at (-0.25,3.3) {$\pmb{1}$}; %Label vertical axiz
                %\node[anchor=north,scale=0.8] (N_label_4) at (2.3,-2.35) {$\pmb{1-b_1}$}; %Label horizontal axiz
                %\node[anchor=north,scale=0.8] (N_label_4) at (3.1,-2.35) {$\pmb{b}$}; %Label horizontal axiz
                %\node[anchor=north,scale=1.0] (N_label_5) at (3.9,-2.2) {$\pmb{\frac{k_1}{n-k_2}}$}; %Label horizontal axiz
                
                \node[anchor=north east,scale=1] (x_axis_label) at (0.05,-2.22) {$\pmb{0}$}; %Label origin
            
                \draw[line width=0.02cm,-{Latex[length=3mm]}] (0,-3) -- (0,3.7);  %draw vertical axes
                \node[scale=1.2,left] (vertaxislabel) at (-0.05,3.6) {$\pmb{q}$}; %label vertical axis

                %preferred locations

                %preferred location labels
                %\fill (0.35,-1.9) circle [radius=0.2em];%(p15,q15)
                \node[anchor=west,scale=0.9] (N_p15q15) at (0.6,-1.5) {$\boldsymbol S_2$}; %S2
                %\node[anchor=west,scale=0.9] (N_p14q14_b) at (1,-0.5) {$\boldsymbol S_2$}; %S2
                 \node[anchor=west,scale=0.9] (N_p2q2) at (1.6,-0.3) {$\boldsymbol S_1$}; %S1
             
            \end{scope}    
        \end{tikzpicture}
    \caption{ \textbf{PK}\ref{PK4}, \textbf{PK}\ref{PK5}, partition the set $\Omega$.
    }
    \label{fig_omega_partition_gen_CBI_simple}
\end{figure}
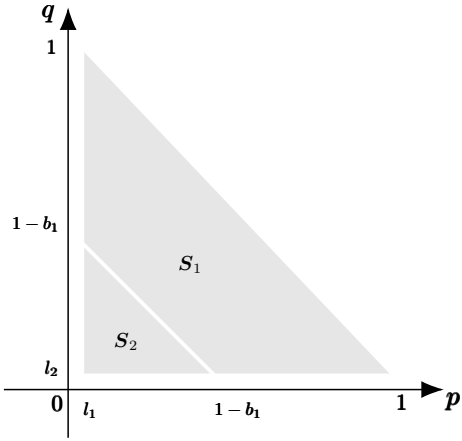

\begin{figure*}[t!]
%\captionsetup[figure]{format=hang}
    %\begin{subfigure}[]{0.47\linewidth}
    \centering
    \subfloat[A discrete prior over $\Omega$ that solves \eqref{eqn_Ber_CBI_problem} when $k>nb$. The black dots indicate $(p_i,q_i)$ locations (in the partition of  Fig.~\ref{fig_omega_partition_Ber_1}) that this prior may assign non-zero probability mass to.\label{fig_preferred_locations_Ber_samePKs}]{
    \begin{minipage}{0.48\linewidth}
    \centering
        \begin{tikzpicture}[scale=\omegaFigScale,
                        every node/.style={scale=\omegaFigScale}] % nodes scale too
            \begin{scope}
                %vertices of sample space triangle with non-zero probability
                \coordinate (r0) at (0.05,-1.25);
                \coordinate (s0) at (0.05,3);
                \coordinate (si) at (5,-2.2);
                \coordinate (ri) at (1,-2.2);
                \fill[fill=gray!20] (r0) -- (s0) -- (si) -- (ri) -- cycle;                
                %k/n q-stationarity line
                \draw[line width=0.05cm,white] (4.35,-2.5) -- (0,1.95);

                %k/n p-stationarity line
                \draw[line width=0.05cm,white] (0,1.4) -- (3.8,-2.5);

                %the b1 line
                %\draw[line width=0.05cm,white] (1.5,-2) -- (1.5,2.75);

                %the b2 line
                %\draw[line width=0.05cm,white] (0.25,-0.9) -- (4.25,-0.9);

                %the b line
                \draw[line width=0.05cm,white] (0,0.25) -- (2.45,-2.2);

                %the 1-b3 line
                \draw[line width=0.05cm,white] (0,-1.25) -- (1,-2.2);

                \draw[-{Latex[length=3mm]},line width=0.02cm] (-1,-2.25) -- (5.85,-2.25); %draw horizontal axis
                \node[anchor=north,scale=1.2] (x_axis_label) at (6,-2.15) {$\pmb{p}$}; %Label horizontal axiz

                %axes labels
                \node[anchor=north,scale=0.8] (N_label_2) at (2.65,-2.35) {$\pmb{1-b_1}$}; %Label horizontal axiz
                \node[anchor=north,scale=0.8] (N_label_4) at (1.25,-2.35) {$\pmb{l}$}; %Label horizontal axiz
                \node[anchor=north,scale=0.8] (N_label_4) at (3.7,-2.35) {$\pmb{b}$}; %Label horizontal axiz
                \node[anchor=north,scale=1.0] (N_label_5) at (4.2,-2.2) {$\pmb{\frac{k}{n}}$}; %Label horizontal axiz
                \node[anchor=north,scale=1.0] (N_label_5) at (5.2,-2.2) {$\pmb{1}$}; %Label horizontal axiz
                \node[anchor=north,scale=0.8] (N_label_2) at (-0.5,0.55) {$\pmb{1-b_1}$}; %Label vertical axiz
                \node[anchor=north,scale=0.8] (N_label_2) at (-0.25,-0.9) {$\pmb{l}$}; %Label vertical axiz
                \node[anchor=north,scale=0.8] (N_label_2) at (-0.25,1.6) {$\pmb{b}$}; %Label vertical axiz
                \node[anchor=north,scale=1] (N_label_2) at (-0.25,2.35) {$\pmb{\frac{k}{n}}$}; %Label vertical axiz
                \node[anchor=north,scale=0.9] (N_label_5) at (-0.25,3.3) {$\pmb{1}$}; %Label vertical axiz
                %\node[anchor=north,scale=0.8] (N_label_4) at (2.3,-2.35) {$\pmb{1-b_1}$}; %Label horizontal axiz
                %\node[anchor=north,scale=0.8] (N_label_4) at (3.1,-2.35) {$\pmb{b}$}; %Label horizontal axiz
                %\node[anchor=north,scale=1.0] (N_label_5) at (3.9,-2.2) {$\pmb{\frac{k_1}{n-k_2}}$}; %Label horizontal axiz
                
                \node[anchor=north east,scale=1] (x_axis_label) at (0.05,-2.22) {$\pmb{0}$}; %Label origin
            
                \draw[line width=0.02cm,-{Latex[length=3mm]}] (0,-3) -- (0,3.7);  %draw vertical axes
                \node[scale=1.2,left] (vertaxislabel) at (-0.05,3.6) {$\pmb{q}$}; %label vertical axis

                %preferred locations
                \fill (0.55,-1.7) circle [radius=0.2em];%(p2,q2)
                \fill (2.15,-0.2) circle [radius=0.2em];%(p1,q1)

                %preferred location labels
                %\fill (0.35,-1.9) circle [radius=0.2em];%(p15,q15)
                \node[anchor=west,scale=0.8] (N_p15q15) at (0.55,-1.5) {$(p_{2},q_{2})$}; %(p2, q2)
                %\fill (0.35,-1.2) circ(2,-1.2)0.2em];%(p14,q14)
                \node[anchor=west,scale=0.8] (N_p2q2) at (2.2,-0.05) {$(p_{1},q_{1})$}; %(p1, q1)

            \end{scope}    
        \end{tikzpicture}
        \end{minipage}}
%\hspace{3cm}
\hfill
%\begin{subfigure}[]{0.47\linewidth}
\subfloat[ A discrete prior over $\Omega$ that solves \eqref{eqn_Ber_CBI_problem_diffPKs} when $k>nb$. The black dots indicate the only $(p_i,q_i)$ locations (in the partition of Fig.~\ref{fig_omega_partition_gen_CBI_simple}) that this prior may assign non-zero probability mass to.\label{fig_preferred_locations_Ber_diffPKs}]{
\begin{minipage}{0.48\linewidth}
\centering
        \begin{tikzpicture}[scale=\omegaFigScale,
                        every node/.style={scale=\omegaFigScale}] % nodes scale too
            \begin{scope}
                %vertices of sample space triangle with non-zero probability
                %\coordinate (r0) at (0.05,-1.25);
                %\coordinate (s0) at (0.05,3);
                %\coordinate (si) at (5,-2.2);
                %\coordinate (ri) at (1,-2.2);
                %\fill[fill=gray!20] (r0) -- (s0) -- (si) -- (ri) -- cycle;

                %vertices of sample space triangle with non-zero probability
                \coordinate (r0) at (0.25,-2);
                \coordinate (s0) at (0.25,3);
                \coordinate (si) at (5,-2);
                %\coordinate (ri) at (1,-2);
                \fill[fill=gray!20] (r0) -- (s0) -- (si) -- cycle;

                %k1/(n-k2) p-stationarity line
                %\draw[line width=0.05cm,white] (0,3) -- (4.3,-2.5);

                \draw[line width=0.05cm,white] (4.35,-2.5) -- (0,1.95);

                %k2/(n-k1) q-stationarity line
                %\draw[line width=0.05cm,white] (4.95,-2.2) -- (0,1.75);
                
                %k/n q-stationarity line
                %\draw[line width=0.05cm,white] (4.35,-2.5) -- (0,1.95);

                %b line
                \draw[line width=0.05cm,white] (0,1.4) -- (3.8,-2.5);

                %the b1 line
                %\draw[line width=0.05cm,white] (1.5,-2) -- (1.5,2.75);

                %the b2 line
                %\draw[line width=0.05cm,white] (0.25,-0.9) -- (4.25,-0.9);

                %the 1-b_1 line
                \draw[line width=0.05cm,white] (0,0.25) -- (2.45,-2.2);

                %the l line
                %\draw[line width=0.05cm,white] (0,-1.25) -- (1,-2.2);

                \draw[-{Latex[length=3mm]},line width=0.02cm] (-1,-2.25) -- (5.85,-2.25); %draw horizontal axis
                \node[anchor=north,scale=1.2] (x_axis_label) at (6,-2.15) {$\pmb{p}$}; %Label horizontal axiz

                %axes labels
                \node[anchor=north,scale=0.8] (N_label_2) at (2.65,-2.35) {$\pmb{1-b_1}$}; %Label horizontal axiz
                \node[anchor=north,scale=0.8] (N_label_2) at (0.35,-2.35) {$\pmb{l_1}$}; %Label horizontal axiz
                %\node[anchor=north,scale=0.8] (N_label_4) at (1.25,-2.35) {$\pmb{l}$}; %Label horizontal axiz
                \node[anchor=north,scale=0.8] (N_label_4) at (3.5,-2.35) {$\pmb{b}$}; %Label horizontal axiz
                \node[anchor=north,scale=1.0] (N_label_5) at (4.2,-2.2) {$\pmb{\frac{k}{n}}$}; %Label horizontal axiz
                \node[anchor=north,scale=1.0] (N_label_5) at (5.2,-2.2) {$\pmb{1}$}; %Label horizontal axiz
                \node[anchor=north,scale=0.8] (N_label_2) at (-0.5,0.55) {$\pmb{1-b_1}$}; %Label vertical axiz
                %\node[anchor=north,scale=0.8] (N_label_2) at (-0.25,-0.9) {$\pmb{l}$}; %Label vertical axiz
                \node[anchor=north,scale=0.8] (N_label_2) at (-0.25,1.6) {$\pmb{b}$}; %Label vertical axiz
                \node[anchor=north,scale=1] (N_label_2) at (-0.25,2.35) {$\pmb{\frac{k}{n}}$}; %Label vertical axiz
                \node[anchor=north,scale=0.9] (N_label_5) at (-0.25,3.3) {$\pmb{1}$}; %Label vertical axiz
                \node[anchor=north,scale=0.8] (N_label_2) at (-0.25,-1.7) {$\pmb{l_2}$}; %Label vertical axiz

                %\node[anchor=north,scale=0.8] (N_label_4) at (2.3,-2.35) {$\pmb{1-b_1}$}; %Label horizontal axiz
                %\node[anchor=north,scale=0.8] (N_label_4) at (3.1,-2.35) {$\pmb{b}$}; %Label horizontal axiz
                %\node[anchor=north,scale=1.0] (N_label_5) at (3.9,-2.2) {$\pmb{\frac{k_1}{n-k_2}}$}; %Label horizontal axiz
                
                \node[anchor=north east,scale=1] (x_axis_label) at (0.05,-2.22) {$\pmb{0}$}; %Label origin
            
                \draw[line width=0.02cm,-{Latex[length=3mm]}] (0,-3) -- (0,3.7);  %draw vertical axes
                \node[scale=1.2,left] (vertaxislabel) at (-0.05,3.6) {$\pmb{q}$}; %label vertical axis

                %preferred locations
                \fill (0.3,-1.95) circle [radius=0.2em];%(p2,q2)
                %\fill (0.0,-1.2) circle [radius=0.2em];%(p2,q2)
                \fill (2.15,-0.2) circle [radius=0.2em];%(p1,q1)

                %preferred location labels
                \node[anchor=south west,scale=0.8] (N_p15q15) at (0.3,-2) {$(p_{2},q_{2})$}; %(p2, q2)
                %\node[anchor=south west,scale=0.8] (N_p15q15) at (0.0,-1.2) {$(p_{2},q_{2})$}; %(p2, q2)
                %\fill (0.35,-1.2) circ(2,-1.2)0.2em];%(p14,q14)
                \node[anchor=west,scale=0.8] (N_p2q2) at (2.2,-0.05) {$(p_{1},q_{1})$}; %(p1, q1)                
             
            \end{scope}    
        \end{tikzpicture}
        \end{minipage}}
%\end{subfigure}
    \caption{Example worst-case prior distributions when the assessor has uncertainty about which failure-types have occurred.}
\label{fig_CBI_Ber_same_vs_diffPKs}
\end{figure*}
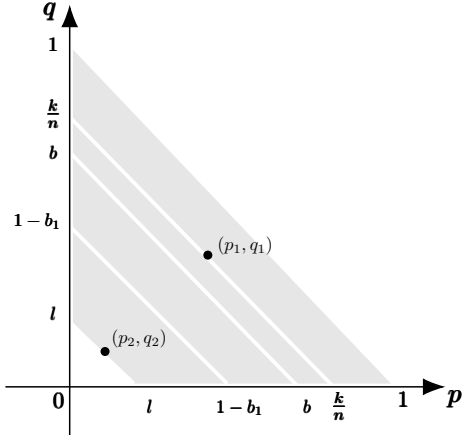
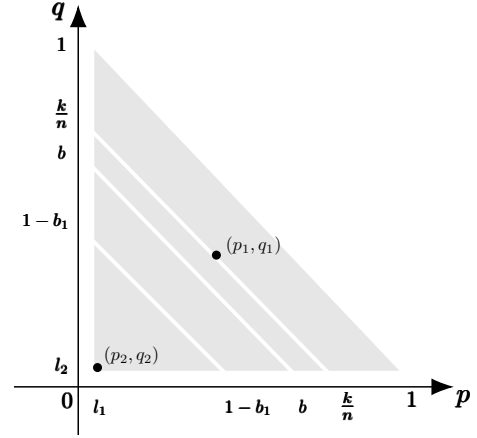

\begin{figure*}[t!]
%\captionsetup[figure]{format=hang}
	%\begin{subfigure}[]{0.48\linewidth}
    %\begin{subfigure}[]{0.47\linewidth}
    \centering
    \subfloat[A discrete prior over $\Omega$ that solves \eqref{eqn_gen_CBI_problem_samePKs} when $k_1+k_2>nb$. The black dots indicate the only $(p_i,q_i)$ locations (in the partition of  Fig.~\ref{fig_omega_partition_Ber_1}) that this prior may assign non-zero probability mass to. Here, $(p_1,q_1)=(\frac{k_1}{n},\frac{k_2}{n})$. Whichever $(p_2,q_2)$ location gives the smallest likelihood value will be assigned non-zero probability $a$.\label{fig_preferred_locations_gen_CBI_samePKs}]{
    \begin{minipage}{0.48\linewidth}
    \centering
        \begin{tikzpicture}[scale=\omegaFigScale,
                        every node/.style={scale=\omegaFigScale}] % nodes scale too
            \begin{scope}
                %vertices of sample space triangle with non-zero probability
                \coordinate (r0) at (0.05,-1.25);
                \coordinate (s0) at (0.05,3);
                \coordinate (si) at (5,-2.2);
                \coordinate (ri) at (1,-2.2);
                \fill[fill=gray!20] (r0) -- (s0) -- (si) -- (ri) -- cycle;                

                %k1/(n-k2) p-stationarity line
                \draw[line width=0.05cm,white] (0,3) -- (4.3,-2.5);

                %k2/(n-k1) q-stationarity line
                \draw[line width=0.05cm,white] (4.95,-2.2) -- (0,1.75);
                
                %k/n q-stationarity line
                %\draw[line width=0.05cm,white] (4.35,-2.5) -- (0,1.95);

                %k/n p-stationarity line
                \draw[line width=0.05cm,white] (0,1.4) -- (3.8,-2.5);

                %the b1 line
                %\draw[line width=0.05cm,white] (1.5,-2) -- (1.5,2.75);

                %the b2 line
                %\draw[line width=0.05cm,white] (0.25,-0.9) -- (4.25,-0.9);

                %the b line
                \draw[line width=0.05cm,white] (0,0.25) -- (2.45,-2.2);

                %the 1-b3 line
                \draw[line width=0.05cm,white] (0,-1.25) -- (1,-2.2);

                \draw[-{Latex[length=3mm]},line width=0.02cm] (-1,-2.25) -- (5.85,-2.25); %draw horizontal axis
                \node[anchor=north,scale=1.2] (x_axis_label) at (6,-2.15) {$\pmb{p}$}; %Label horizontal axiz

                %axes labels
                \node[anchor=north,scale=0.8] (N_label_2) at (2.65,-2.35) {$\pmb{1-b_1}$}; %Label horizontal axiz
                \node[anchor=north,scale=0.8] (N_label_4) at (1.25,-2.35) {$\pmb{l}$}; %Label horizontal axiz
                \node[anchor=north,scale=0.8] (N_label_4) at (3.5,-2.35) {$\pmb{b}$}; %Label horizontal axiz
                \node[anchor=north,scale=1.0] (N_label_5) at (4.2,-2.2) {$\pmb{\frac{k_1}{n-k_2}}$}; %Label horizontal axiz
                \node[anchor=north,scale=1.0] (N_label_5) at (5.2,-2.2) {$\pmb{1}$}; %Label horizontal axiz
                \node[anchor=north,scale=0.8] (N_label_2) at (-0.5,0.55) {$\pmb{1-b_1}$}; %Label vertical axiz
                \node[anchor=north,scale=0.8] (N_label_2) at (-0.25,-0.9) {$\pmb{l}$}; %Label vertical axiz
                \node[anchor=north,scale=0.8] (N_label_2) at (-0.25,1.6) {$\pmb{b}$}; %Label vertical axiz
                \node[anchor=north,scale=1] (N_label_2) at (-0.5,2.35) {$\pmb{\frac{k_2}{n-k_1}}$}; %Label vertical axiz
                \node[anchor=north,scale=0.9] (N_label_5) at (-0.25,3.3) {$\pmb{1}$}; %Label vertical axiz
                %\node[anchor=north,scale=0.8] (N_label_4) at (2.3,-2.35) {$\pmb{1-b_1}$}; %Label horizontal axiz
                %\node[anchor=north,scale=0.8] (N_label_4) at (3.1,-2.35) {$\pmb{b}$}; %Label horizontal axiz
                %\node[anchor=north,scale=1.0] (N_label_5) at (3.9,-2.2) {$\pmb{\frac{k_1}{n-k_2}}$}; %Label horizontal axiz
                
                \node[anchor=north east,scale=1] (x_axis_label) at (0.05,-2.22) {$\pmb{0}$}; %Label origin
            
                \draw[line width=0.02cm,-{Latex[length=3mm]}] (0,-3) -- (0,3.7);  %draw vertical axes
                \node[scale=1.2,left] (vertaxislabel) at (-0.05,3.6) {$\pmb{q}$}; %label vertical axis

                %preferred locations
                \fill (1.1,-2.25) circle [radius=0.2em];%(p2,q2)
                \fill (0.0,-1.2) circle [radius=0.2em];%(p2,q2)
                \fill (2.55,-0.3) circle [radius=0.2em];%(p1,q1)

                %preferred location labels
                \node[anchor=south west,scale=0.8] (N_p15q15) at (1.1,-2.25) {$(p_{2},q_{2})$}; %(p2, q2)
                \node[anchor=south west,scale=0.8] (N_p15q15) at (0.0,-1.2) {$(p_{2},q_{2})$}; %(p2, q2)
                %\fill (0.35,-1.2) circ(2,-1.2)0.2em];%(p14,q14)
                \node[anchor=south west,scale=0.8] (N_p2q2) at (2.55,-0.3) {$(p_{1},q_{1})$}; %(p1, q1)
             
            \end{scope}    
        \end{tikzpicture}
        \end{minipage}}
%\end{subfigure}
\hfill
%\hspace{3cm}
%\begin{subfigure}[]{0.47\linewidth}
\subfloat[A discrete prior over $\Omega$ that solves \eqref{eqn_gen_CBI_problem_diffPKs} when $k_1+k_2>nb$. The black dots indicate the only $(p_i,q_i)$ locations (in the partition of  Fig.~\ref{fig_omega_partition_gen_CBI_simple}) that this prior may assign non-zero probability mass to. Here, $(p_1,q_1)=(\frac{k_1}{n},\frac{k_2}{n})$.\label{fig_preferred_locations_gen_CBI_1}]{
\begin{minipage}{0.48\linewidth}
\centering
        \begin{tikzpicture}[scale=\omegaFigScale,
                        every node/.style={scale=\omegaFigScale}] % nodes scale too
            \begin{scope}
                %vertices of sample space triangle with non-zero probability
                %\coordinate (r0) at (0.05,-1.25);
                %\coordinate (s0) at (0.05,3);
                %\coordinate (si) at (5,-2.2);
                %\coordinate (ri) at (1,-2.2);
                %\fill[fill=gray!20] (r0) -- (s0) -- (si) -- (ri) -- cycle;

                %vertices of sample space triangle with non-zero probability
                \coordinate (r0) at (0.25,-2);
                \coordinate (s0) at (0.25,3);
                \coordinate (si) at (5,-2);
                %\coordinate (ri) at (1,-2);
                \fill[fill=gray!20] (r0) -- (s0) -- (si) -- cycle;

                %k1/(n-k2) p-stationarity line
                \draw[line width=0.05cm,white] (0,3) -- (4.3,-2.5);

                %k2/(n-k1) q-stationarity line
                \draw[line width=0.05cm,white] (4.95,-2.2) -- (0,1.75);
                
                %k/n q-stationarity line
                %\draw[line width=0.05cm,white] (4.35,-2.5) -- (0,1.95);

                %b line
                \draw[line width=0.05cm,white] (0,1.4) -- (3.8,-2.5);

                %the b1 line
                %\draw[line width=0.05cm,white] (1.5,-2) -- (1.5,2.75);

                %the b2 line
                %\draw[line width=0.05cm,white] (0.25,-0.9) -- (4.25,-0.9);

                %the 1-b_1 line
                \draw[line width=0.05cm,white] (0,0.25) -- (2.45,-2.2);

                %the l line
                %\draw[line width=0.05cm,white] (0,-1.25) -- (1,-2.2);

                \draw[-{Latex[length=3mm]},line width=0.02cm] (-1,-2.25) -- (5.85,-2.25); %draw horizontal axis
                \node[anchor=north,scale=1.2] (x_axis_label) at (6,-2.15) {$\pmb{p}$}; %Label horizontal axiz

                %axes labels
                \node[anchor=north,scale=0.8] (N_label_2) at (2.65,-2.35) {$\pmb{1-b_1}$}; %Label horizontal axiz
                \node[anchor=north,scale=0.8] (N_label_2) at (0.35,-2.35) {$\pmb{l_1}$}; %Label horizontal axiz
                %\node[anchor=north,scale=0.8] (N_label_4) at (1.25,-2.35) {$\pmb{l}$}; %Label horizontal axiz
                \node[anchor=north,scale=0.8] (N_label_4) at (3.5,-2.35) {$\pmb{b}$}; %Label horizontal axiz
                \node[anchor=north,scale=1.0] (N_label_5) at (4.2,-2.2) {$\pmb{\frac{k_1}{n-k_2}}$}; %Label horizontal axiz
                \node[anchor=north,scale=1.0] (N_label_5) at (5.2,-2.2) {$\pmb{1}$}; %Label horizontal axiz
                \node[anchor=north,scale=0.8] (N_label_2) at (-0.5,0.55) {$\pmb{1-b_1}$}; %Label vertical axiz
                %\node[anchor=north,scale=0.8] (N_label_2) at (-0.25,-0.9) {$\pmb{l}$}; %Label vertical axiz
                \node[anchor=north,scale=0.8] (N_label_2) at (-0.25,1.6) {$\pmb{b}$}; %Label vertical axiz
                \node[anchor=north,scale=1] (N_label_2) at (-0.5,2.35) {$\pmb{\frac{k_2}{n-k_1}}$}; %Label vertical axiz
                \node[anchor=north,scale=0.9] (N_label_5) at (-0.25,3.3) {$\pmb{1}$}; %Label vertical axiz
                \node[anchor=north,scale=0.8] (N_label_2) at (-0.25,-1.7) {$\pmb{l_2}$}; %Label vertical axiz

                %\node[anchor=north,scale=0.8] (N_label_4) at (2.3,-2.35) {$\pmb{1-b_1}$}; %Label horizontal axiz
                %\node[anchor=north,scale=0.8] (N_label_4) at (3.1,-2.35) {$\pmb{b}$}; %Label horizontal axiz
                %\node[anchor=north,scale=1.0] (N_label_5) at (3.9,-2.2) {$\pmb{\frac{k_1}{n-k_2}}$}; %Label horizontal axiz
                
                \node[anchor=north east,scale=1] (x_axis_label) at (0.05,-2.22) {$\pmb{0}$}; %Label origin
            
                \draw[line width=0.02cm,-{Latex[length=3mm]}] (0,-3) -- (0,3.7);  %draw vertical axes
                \node[scale=1.2,left] (vertaxislabel) at (-0.05,3.6) {$\pmb{q}$}; %label vertical axis

                %preferred locations
                \fill (0.3,-1.95) circle [radius=0.2em];%(p2,q2)
                %\fill (0.0,-1.2) circle [radius=0.2em];%(p2,q2)
                \fill (2.55,-0.3) circle [radius=0.2em];%(p1,q1)

                %preferred location labels
                \node[anchor=south west,scale=0.8] (N_p15q15) at (0.3,-2) {$(p_{2},q_{2})$}; %(p2, q2)
                %\node[anchor=south west,scale=0.8] (N_p15q15) at (0.0,-1.2) {$(p_{2},q_{2})$}; %(p2, q2)
                %\fill (0.35,-1.2) circ(2,-1.2)0.2em];%(p14,q14)
                \node[anchor=south west,scale=0.8] (N_p2q2) at (2.55,-0.3) {$(p_{1},q_{1})$}; %(p1, q1)
             
            \end{scope}    
        \end{tikzpicture}
        \end{minipage}}
%\end{subfigure}
    \caption{ Example worst-case prior distributions when the assessor knows which failure-types have occurred.}
\label{fig_CBI_gen_same_vs_diffPKs}
\end{figure*}
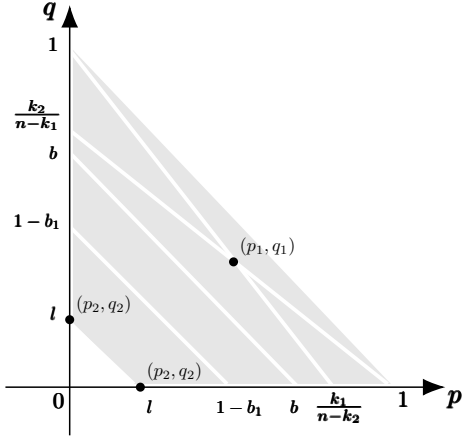
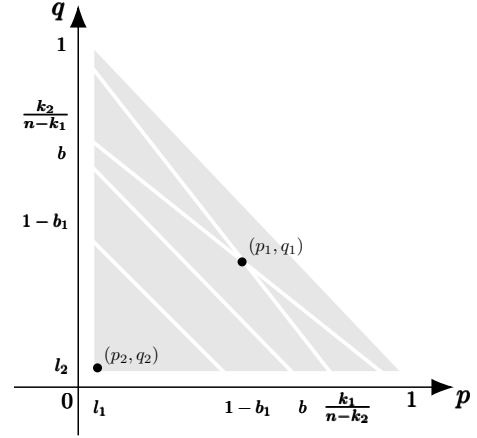

\subsection{The Feasible Values for FN and FP Probabilities}
The feasible $(P,Q)$-values lie within a triangle, $\Omega$, in the unit square $[0,1]^2$. Two partitions of $\Omega$ are relevant for the assessment scenarios we consider. Fig.~\ref{fig_omega_partition_Ber_1} is the partition due to beliefs \textbf{PK}\ref{PK1} and \textbf{PK}\ref{PK4}, and Fig.~\ref{fig_omega_partition_gen_CBI_simple} is the partition due to \textbf{PK}\ref{PK4} and \textbf{PK}\ref{PK5}. Depending on the PKs, $\mathcal D$ contains distributions that assign probabilities to subsets in Fig.~\ref{fig_omega_partition_Ber_1} or Fig.~\ref{fig_omega_partition_gen_CBI_simple}.

\section{Confidence Bounds on Reliability}
\label{sec_results}
We present four theorems that give conservative confidence in an upper bound $b$ on the \emph{pfc} $P+Q$. Recall, accuracy $\Theta$ is $1-\textit{pfc}$. For ease of presentation, Theorems~\ref{Thrm_CBI_Bernoulli}--\ref{Thrm_gen_CBI_diffPKs} give the assessor's smallest ``confidence'' (i.e. infima), rather than the equivalent assessor's greatest doubt (i.e. suprema) in \eqref{eqn_lessabstract_CBI_problem}, \eqref{eqn_Bernoulli_CBI_problem}. 

In what follows, if $L_i$ is the value of the relevant likelihood at point $(p_i,q_i)\in{\boldsymbol S}_i$ of the relevant partition of $\Omega$, define\footnote{$L^*_i$ or $L_{i*}$ are relevant only when their ``${\boldsymbol S}_i\cap...$'' are non-empty.} 
\begin{gather*}
L^*_i:=\sup\limits_{{\boldsymbol S}_i\cap [P+Q> b]}L_i\,,\quad L_{i*}:=\inf\limits_{{\boldsymbol S}_i\cap [P+Q\leqslant b]}L_i
\end{gather*}

\subsection{Assessor is Unaware of Failure--Type Occurrences}
Consider the problem \eqref{eqn_Bernoulli_CBI_problem} of determining the least confidence (in the software's \emph{pfc} being small enough), when the assessor knows $k$ failures occurred (over $n$ classifications) but does not know each failure's type. Theorem~\ref{Thrm_CBI_Bernoulli} solves this.

\begin{theorem}
\label{Thrm_CBI_Bernoulli}

%\noindent\textbf{Problem:}
\noindent For $n,\,k$, $a$, $b$, $b_1$, $l$, $P$, $Q$, $\mathcal{D}$ as already defined and constrained, we seek 
\begin{align} &\underset{\mathcal{D}}{\inf}\, \dfrac{\mathbb E[L(n,k; P, Q)\mathbf{1}_{P+Q\leqslant b}]}{\mathbb E[L(n,k; P, Q)]}   \label{eqn_Ber_CBI_problem}\\
    s.t. \,\,&\text{\bf PK}\ref{PK1},\,\text{\bf PK}\ref{PK4}  \nonumber 
\end{align}

\noindent\textbf{Solution:} 
Let $L_i:=L(n,k; p_i, q_i)$ for $(p_i,q_i)\in {\boldsymbol S_i}$ in Fig.~\ref{fig_omega_partition_Ber_1}. The infimum $\Phi^*$ takes the form
\begin{align*}
 \Phi^*=&\frac{aL_{2*}}{aL_{2*}+(1-a)L^*_{1}}{\boldsymbol 1}_{1-b_1\leqslant b}
\end{align*}
That is, objective function \eqref{eqn_Ber_CBI_problem} attains its infimum, $\Phi^*$, with  a discrete prior distribution over $\Omega$ of the form (e.g. see Fig.~\ref{fig_preferred_locations_Ber_samePKs}):   
\begin{align*}
    \mathbb P(P=p_i,Q=q_i)=\begin{cases}
    1-a, & \text{ if } i=1\\
    a, & \text{ if } i=2%;\\
           %0, & \text{ otherwise}
    \end{cases}
\end{align*}
 %and the $p_i$s and $q_i$s are placed at preferred locations, so that $\Phi$ is maximized.
\end{theorem}
Our second theorem also solves the problem, but the assessor expresses the more detailed \textbf{PK}\ref{PK5} belief instead of \textbf{PK}\ref{PK1}.

\begin{theorem}
\label{Thrm_CBI_Bernoulli_diffPKS}
%\noindent\textbf{Problem:}
\noindent For $n,\,k$, $a$, $b$, $b_1$, $l_1$, $l_2$, $P$, $Q$, $\mathcal{D}$ as already defined and constrained, we seek 
\begin{align} &\underset{\mathcal{D}}{\inf}\, \dfrac{\mathbb E[L(n,k; P, Q)\mathbf{1}_{P+Q\leqslant b}]}{\mathbb E[L(n,k; P, Q)]}   \label{eqn_Ber_CBI_problem_diffPKs}\\
    s.t. \,\,&\text{\bf PK}\ref{PK4},\,\text{\bf PK}\ref{PK5}  \nonumber 
\end{align}

\noindent\textbf{Solution:} 
Let $L_i:=L(n,k; p_i, q_i)$ for $(p_i,q_i)\in {\boldsymbol S_i}$ in Fig.~\ref{fig_omega_partition_gen_CBI_simple}. The infimum $\Phi^*$ takes the form
\begin{align*}
 \Phi^*=&\frac{aL_{2*}}{aL_{2*}+(1-a)L^*_{1}}{\boldsymbol 1}_{1-b_1\leqslant b}
\end{align*}
That is, objective function \eqref{eqn_Ber_CBI_problem_diffPKs} attains its infimum, $\Phi^*$, with  a discrete prior distribution over $\Omega$ of the form (e.g. see Fig.~\ref{fig_preferred_locations_Ber_diffPKs}):    
\begin{align*}
    \mathbb P(P=p_i,Q=q_i)=\begin{cases}
    1-a, & \text{ if } i=1\\
    a, & \text{ if } i=2%;\\
           %0, & \text{ otherwise}
    \end{cases}
\end{align*}
\end{theorem}

\subsection{Assessor is Aware of Failure--Type Occurrences}
Now consider problem \eqref{eqn_lessabstract_CBI_problem} of determining the least confidence (in the \emph{pfc} being small enough) an assessor can justify, when the assessor knows $k_1$ FPs and $k_2$ FNs have occurred in $n$ classifications. Our third theorem is analogous to Theorem~\ref{Thrm_CBI_Bernoulli}.

\begin{theorem}
\label{Thrm_gen_CBI_samePKs}
%\noindent\textbf{Problem:}

\noindent For $n,\,k_1,\,k_2$, $a$, $b$, $b_1$, $l$, $P$, $Q$, $\mathcal{D}$ as already defined and constrained, we seek 
\begin{align} &\underset{\mathcal{D}}{\inf}\, \dfrac{\mathbb E[L(n,k_1,k_2; P, Q)\mathbf{1}_{P+Q\leqslant b}]}{\mathbb E[L(n,k_1,k_2; P, Q)]}   \label{eqn_gen_CBI_problem_samePKs}\\
    s.t. \,\,&\text{\bf PK}\ref{PK1},\,\text{\bf PK}\ref{PK4}  \nonumber 
\end{align}

\noindent\textbf{Solution:} 
\noindent Let  $L_{i}:=L(n,k_1,k_2; p_{i},q_{i})$ for $(p_{i},q_{i})\in {\boldsymbol S_{i}}$ in Fig.~\ref{fig_omega_partition_Ber_1}. The infimum $\Phi^*$ is $0$ if $k_1\geqslant 1$ or $k_2\geqslant 1$. Otherwise,
\begin{align*}
 \Phi^*=&\frac{aL_{2*}}{aL_{2*}+(1-a)L^*_{1}}{\boldsymbol 1}_{1-b_1\leqslant b,\,k_1=k_2=0}
\end{align*}

That is, objective function \eqref{eqn_gen_CBI_problem_samePKs} attains its infimum, $\Phi^*$, with  a discrete prior distribution over $\Omega$ of the form (e.g. see Fig.~\ref{fig_preferred_locations_gen_CBI_samePKs}):  
\begin{align*}
    \mathbb P(P=p_i,Q=q_i)=\begin{cases}
    1-a, & \text{ if } i=1\\
    a, & \text{ if } i=2%;\\
           %0, & \text{ otherwise}
    \end{cases}
\end{align*}
 %and the $p_i$s and $q_i$s are placed at preferred locations, so that $\Phi$ is maximized.
\end{theorem}

Like Theorem~\ref{Thrm_gen_CBI_samePKs}, the next theorem also solves CBI problem \eqref{eqn_lessabstract_CBI_problem}, but with \textbf{PK}\ref{PK1} replaced by the more detailed \textbf{PK}\ref{PK5}.

\begin{theorem}
\label{Thrm_gen_CBI_diffPKs}
%\noindent\textbf{Problem:}

\noindent For $n,\,k_1,\,k_2$, $a$, $b$, $b_1$,  $l_1$, $l_2$, $P$, $Q$, $\mathcal{D}$ as already defined and constrained, we seek 
\begin{align} &\underset{\mathcal{D}}{\inf}\, \dfrac{\mathbb E[L(n,k_1,k_2; P, Q)\mathbf{1}_{P+Q\leqslant b}]}{\mathbb E[L(n,k_1,k_2; P, Q)]}   \label{eqn_gen_CBI_problem_diffPKs}\\
    s.t. \,\,&\text{\bf PK}\ref{PK4},\,\text{\bf PK}\ref{PK5}  \nonumber 
\end{align}

\noindent\textbf{Solution:} 
\noindent Let  $L_{i}:=L(n,k_1,k_2; p_{i},q_{i})$ for $(p_{i},q_{i})\in {\boldsymbol S_{i}}$ in Fig.~\ref{fig_omega_partition_gen_CBI_simple}. The infimum $\Phi^*$ takes the form
\begin{align*}
 \Phi^*=&\frac{aL_{2*}}{aL_{2*}+(1-a)L^*_{1}}{\boldsymbol 1}_{1-b_1\leqslant b}
\end{align*}
That is, objective function \eqref{eqn_gen_CBI_problem_diffPKs} attains its infimum, $\Phi^*$, with  a discrete prior distribution over $\Omega$ of the form (e.g. see Fig.~\ref{fig_preferred_locations_gen_CBI_1}):    
\begin{align*}
    \mathbb P(P=p_i,Q=q_i)=\begin{cases}
    1-a, & \text{ if } i=1\\
    a, & \text{ if } i=2%;\\
           %0, & \text{ otherwise}
    \end{cases}
\end{align*}
\end{theorem}

Theorem~\ref{Thrm_gen_CBI_diffPKs} is proved in Appendix~\ref{subsec_proof_Theorem4}; the other theorems are proved using analogous steps. The $\Phi^\ast$ solutions for \eqref{eqn_Ber_CBI_problem} through  \eqref{eqn_gen_CBI_problem_diffPKs} may have the same functional form, but not necessarily the same values; the locations $(p_i,q_i)$ and likelihood $L_i$ used for each $\Phi^\ast$ differ, depending on the objective function, \textbf{PK}s, and parameter values. Figs.~\ref{fig_CBI_Ber_same_vs_diffPKs}--\ref{fig_CBI_gen_same_vs_diffPKs} show examples of ``\emph{worst-case}'' prior distributions that solve \eqref{eqn_Ber_CBI_problem}--\eqref{eqn_gen_CBI_problem_diffPKs} for particular parameter ranges. These discrete distributions are depicted ``from above'', looking down on the distributions and their common domain $\Omega$. Each black dot is a location---in the relevant partition of $\Omega$ (i.e. Figs.~\ref{fig_omega_partition_Ber_1}, \ref{fig_omega_partition_gen_CBI_simple})---that is assigned a probability mass consistent with the assessor's \textbf{PK} beliefs stated in the relevant theorem.  

These results extend previous CBI solutions. In the appropriate limits, the marginals of these two dimensional worst--case priors recover the one dimensional CBI worst--case priors in \cite{zhao_assessing_2019}; e.g. when no failures are observed (see Section~\ref{sec_applications}).

%The forms of these worst-case joint priors are governed by the MLE for each relevant likelihood: in summary, the MLE acts as an ``attractor'' of preferred locations (for those ${\boldsymbol S}_i$ subsets that lie above the $b$ line) or "repeller" of preferred locations (for subsets that lie below the $b$ line). See Fig.~17 in Appendix~A. Such governing points are a general structural property of  worst-case/conservative priors that explains \emph{how}/\emph{why} worst-case priors achieve conservative results \cite{SalakoMuhammad2025}.

Strictly speaking, worst-case priors may not be feasible, because they may assign probability mass to locations that violate one or more \textbf{PK}s. Instead, worst-case priors should be viewed as limits of feasible priors, and their preferred locations as so-called \emph{limit points} of the respective ${\boldsymbol S}_i$ subsets (see \cite{salako2026fixedpointcharacterisationsextremaldistributions}). 

%\subsection{Conservative Assessor Beliefs}
%\label{subsec_consasessorbeliefs}
Each Theorem's worst--case prior encodes the most skeptical beliefs an assessor can have; beliefs about what the unknown probabilities, $P$, $Q$, might be. Figs.~\ref{fig_CBI_Ber_same_vs_diffPKs}--\ref{fig_CBI_gen_same_vs_diffPKs} show that the ``\emph{most skeptical}'' assessor believes the observed failure behavior of the software can arise in two ways: either the software has failure probability values $(p_1,q_1)$ that make this behavior \emph{as likely as possible}\footnote{That is, the $(P,Q)$-value where the likelihood is biggest.} when the software \emph{is not} sufficiently accurate (i.e. when $[P+Q> b]$), or the software has failure probability values $(p_2, q_2)$ that make this behavior \emph{as unlikely as possible} when the software \emph{is} sufficiently accurate (i.e. when $[P+Q\leqslant b]$). This extreme skepticism depends on the software's observed behavior, so an assessor cannot choose to be this skeptical before observing the software in operation. But, the point of these worst--case priors is not that an assessor \emph{should} hold these beliefs, but that an assessor \emph{could} have held these beliefs before observing the software (since these beliefs are consistent with their \textbf{PK}s). And, \emph{had} they held these beliefs, their doubts about the software's accuracy (after observing the software) would nevertheless be justified.

Which of Theorems~\ref{Thrm_CBI_Bernoulli}--\ref{Thrm_gen_CBI_diffPKs} an assessor uses is determined by the fidelity of the available operational observations and the prior-knowledge constraints that can be justified. In particular, Theorems~\ref{Thrm_CBI_Bernoulli} and \ref{Thrm_CBI_Bernoulli_diffPKS} address observations for which failure types are unknown, whereas Theorems~\ref{Thrm_gen_CBI_samePKs} and \ref{Thrm_gen_CBI_diffPKs} incorporate separate FP and FN counts. Within each pair, the applicable theorem depends on whether the assessor's evidence supports \textbf{PK}\ref{PK1} or the more restrictive \textbf{PK}\ref{PK5}, together with \textbf{PK}\ref{PK4}. %Thus, theorem selection is not an independent choice made according to the resulting confidence bound; it follows from the observational model and the admissible prior class established for the assessment.

\section{Application to AV Safety Assessment}
\label{sec_applications}

The following assessment scenarios, adapted from \cite{zhao_assessing_2019}, illustrate the impact of not accounting for multiple failure types. The classifier-level reliability targets and confidence criteria are illustrative; they are neither derived directly from ASIL D nor calibrated to vehicle-level hazard rates.
%The following assessment scenarios adapted from \cite{zhao_assessing_2019} illustrate the impact of not accounting for multiple failure--types. %In each scenario, AV safety software is required to interpret streams of sensor/visual data and determine whether the AV has entered into an unsafe driving state.  
\begin{scenario}[no failures]
\label{example_nofailures}     
Consider an AV safety-monitoring function (e.g. an unsafe-driving-state detector in the monitoring channel) that contributes to an \emph{ASIL~D} safety goal in an ISO~26262 safety case. Suppose an assessor is required to have at least $95\%$ \emph{posterior confidence} that the monitor's pfc satisfies an upper bound $b$. The assessor's pre-operational knowledge is encoded in \textbf{PK}s, including that neither failure mode can be arbitrarily unlikely (e.g.\ $P,Q\geqslant 10^{-15}$) and, prior to observing operational data, the assessor is $90\%$ confident that the classifier meets the development accuracy target $\Theta\geqslant b_1=1-1.09\times 10^{-10}$. The AV then accumulates $n$ classifications in a specific operational slice without failure.
\end{scenario}
\begin{figure}[t!]
	\centering
    \includegraphics[width=0.9\linewidth]{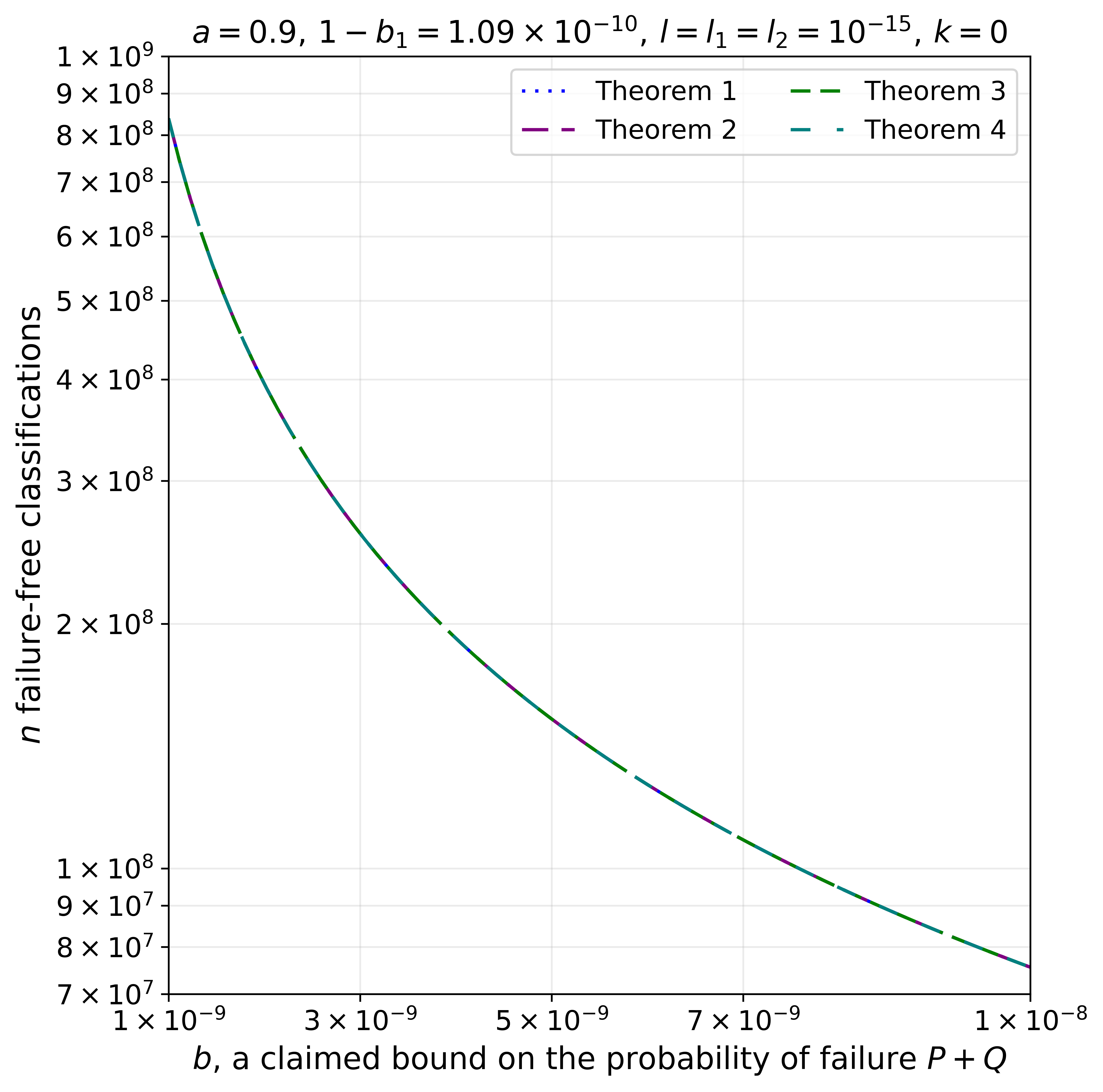}
	\caption{  The number of failure-free classifications needed for $95\%$ confidence in the software accuracy satisfying $\Theta\geqslant 1-b$.}
	\label{fig_NumfailurefreeClassNeeded_UHRel}
\end{figure}
\textbf{How many failure-free classifications are required to, conservatively, attain the required $\boldsymbol{95\%}$ confidence level?}  Fig.~\ref{fig_NumfailurefreeClassNeeded_UHRel} illustrates how the required number of failure-free classifications is \emph{the same} for different conservative assessors. The number is unaffected by differences  between \textbf{PK}\ref{PK1} and the ``more restrictive'' \textbf{PK}\ref{PK5}, nor by the possibility of different failure types occurring. This is consistent with \cite{zhao_assessing_2019}. Consequently, the different theorems' curves overlap perfectly. %This agreement still holds if the assessors are interested in less stringent \emph{pfc} claims, as depicted in Fig.~\ref{fig_NumfailurefreeClassNeeded_notUHRel}. 
%In contrast, what happens when a single failure is observed?    

%\begin{figure}[t!]
%	\centering
	%\includegraphics[width=1.0\linewidth]{Images/fig_NumfailurefreeClassNeeded_notUHRel.png}
%    \includegraphics[width=1.0\linewidth]{Images/Fig11.png}
%	\caption{\textit{The number of failure-free classifications needed to justify $95\%$ confidence in the software accuracy satisfying $\Theta\geqslant 1-b$.}}
%	\label{fig_NumfailurefreeClassNeeded_notUHRel}
%\end{figure}

\begin{scenario}[a single failure]
\label{example_1failure}
Assume Scenario~\ref{example_nofailures} and suppose that after $n_1$ failure-free classifications the conservative posterior confidence meets the $95\%$-confidence acceptance criterion. The monitor then fails on the $(n_1+1)$-th classification, causing the assessor's confidence to drop. 
\end{scenario}

\textbf{How many more failure-free classifications, $\boldsymbol n_2$, must the software complete in order for the assessor's confidence to return to $\boldsymbol{95\%}$?} It depends on the assessor's prior knowledge and whether the types of failures that occur during operation are known. Fig.~\ref{fig_n2vsn1_comparisons} compares the $(n_1, n_2)$ values for different conservative assessors. No curve is plotted for Theorem~\ref{Thrm_gen_CBI_samePKs} because this assessor's confidence \emph{drops to zero upon seeing a single failure}---for example, if $k_1=1$ and $k_2=0$, then the posterior confidence \eqref{eqn_gen_CBI_problem_samePKs} is zero, using the prior in Fig.~\ref{fig_preferred_locations_gen_CBI_samePKs} with $(p_2,q_2)=(0,l)$. Consequently, after a single failure, no finite amount of failure-free operational evidence will suffice for the assessor to regain their lost confidence: $n_2$ for Theorem~\ref{Thrm_gen_CBI_samePKs} is infinity, dwarfing the other theorems' curves.

The infinite $n_2$ from Theorem~\ref{Thrm_gen_CBI_samePKs} is due to \textbf{PK}\ref{PK1}---it allows for the conservative belief that failures consisting of only one type are extremely unlikely, \emph{when the pfc is smaller than $b$}. Using \textbf{PK}\ref{PK5} instead, in Theorem~\ref{Thrm_gen_CBI_diffPKs}, gives finite $n_2$ for all $n_1$.

The larger $n_1$ is, the smaller the bound $b$ that $n_1$ ultimately justifies having $95\%$ confidence in. So, it is unsurprising that the number $n_2$ of required additional classifications also increases. This is apparent from Theorem~\ref{Thrm_gen_CBI_diffPKs}'s monotonically increasing curve, but less apparent from Theorem~\ref{Thrm_CBI_Bernoulli} and \ref{Thrm_CBI_Bernoulli_diffPKS}'s curves---for all sufficiently large $n_1$, these two curves are monotonically-increasing and bounded. The shapes of Theorem~\ref{Thrm_CBI_Bernoulli} and \ref{Thrm_CBI_Bernoulli_diffPKS}'s curves mirror the analogous univariate CBI problem 
%: see Sec.~III-B and App.~B of 
in \cite{zhao_assessing_2019}. The ``hump'' is a result of a change in the preferred location in ${\boldsymbol S}_2$  (Figs.~\ref{fig_omega_partition_Ber_1}, \ref{fig_omega_partition_gen_CBI_simple}) at about $n_1\approx 10^{11}$ and  $n_2\overset{n_1\rightarrow\infty}{\longrightarrow}\lceil\frac{b_1}{1-b_1}\rceil$; Appendix~\ref{app_B} has a proof adapted from \cite{zhao_assessing_2019}. 

\begin{figure}[t!]
	\centering
	\includegraphics[width=0.9\linewidth]{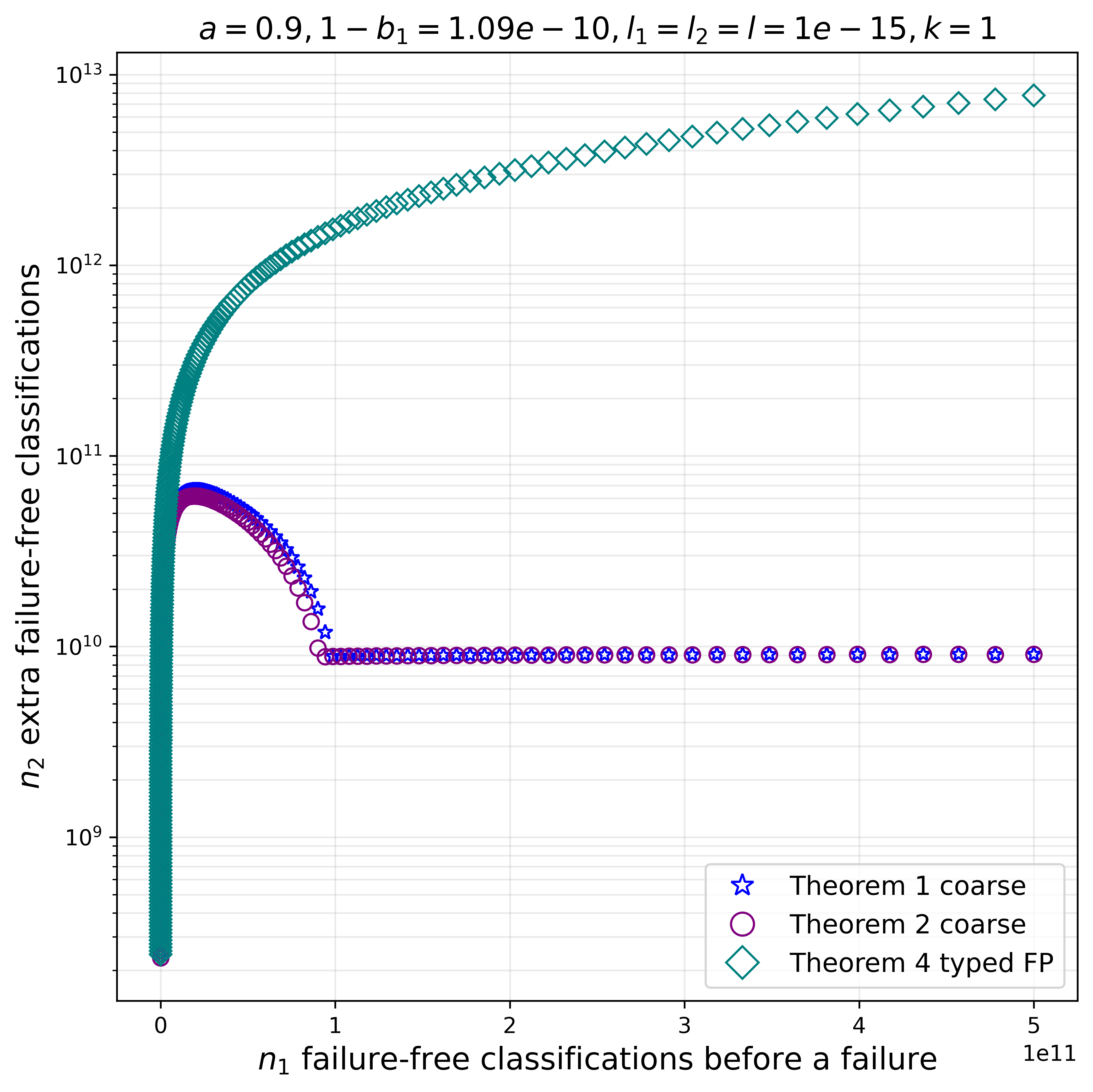}
	\caption{ The extra failure-free classifications needed to regain lost confidence from a single failure. For each $n_1$, the bound $b$ is deduced---all four theorems give the same $b$ under failure-free operation---from which $n_2$ is determined. Theorem~\ref{Thrm_gen_CBI_samePKs} is omitted because $n_2=\infty$ for all $n_1$ (see main text for details).}
	\label{fig_n2vsn1_comparisons}
\end{figure}

\section{Discussion}
\label{sec_discussion}
\subsection{Impact of Uncertainty from Failure--data Fidelity}
\label{subsec_ImpactofFidelity}
Uncertainty about failure--data does not necessarily undermine conservatism in assessments---but, when it does, this can be significant. When no failures are observed, Fig.~\ref{fig_NumfailurefreeClassNeeded_UHRel} shows that Theorems~\ref{Thrm_CBI_Bernoulli}, \ref{Thrm_CBI_Bernoulli_diffPKS}, \ref{Thrm_gen_CBI_samePKs}, \ref{Thrm_gen_CBI_diffPKs} are all equally conservative, despite ignorance about which failure types have occurred in Theorems~\ref{Thrm_CBI_Bernoulli}, \ref{Thrm_CBI_Bernoulli_diffPKS}. This is consistent with our analytical focus on the impact of failure-type uncertainty. However, when failures \emph{are} observed, Fig.~\ref{fig_n2vsn1_comparisons} depicts several orders of magnitude differences between the theorems' results. After a single failure, the ``optimism gap'' between Theorem~\ref{Thrm_gen_CBI_samePKs} and the other three theorems is infinite, as no finite amount of additional failure--free operation can restore 95\% confidence. The gap between Theorem~\ref{Thrm_gen_CBI_diffPKs} and the remaining two theorems grows in the number of extra successes $n_2$ needed following an initial $n_1$ observed successes and a single failure: at $n_1\approx 9\times 10^{11}$, Theorem~\ref{Thrm_gen_CBI_diffPKs} requires $n_2\approx 10^{13}$ which is three orders of magnitude greater than the alternative theorems. These results demonstrate that not accounting for multiple failure modes can lead to concluding that confidence can be recovered after a failure and the evidence requirements to do this are modest, whereas a mode--aware conservative analysis can conclude recovery is vastly more difficult or, even, infeasible. Also note, Theorem~\ref{Thrm_gen_CBI_diffPKs}'s curve is a worst-case result, so the vertical ``gap'' between its curve and the other curves represents posterior confidence from \emph{infinitely many} non-conservative prior distributions that require more failure-free operation (when failure-types are accounted for) than the failure-free operation required by Theorems~\ref{Thrm_CBI_Bernoulli}, \ref{Thrm_CBI_Bernoulli_diffPKS}.

Although incorporating multiple failure modes can lead to more conservative conclusions, Theorems~\ref{Thrm_gen_CBI_samePKs} and \ref{Thrm_gen_CBI_diffPKs} have different sensitivities to observed failures. These differences arise from the relative evidential strength of \textbf{PK}\ref{PK1} and \textbf{PK}\ref{PK5}: the \textbf{PK}s represent evidence that can, or cannot, rule out extremely undesirable failure behavior. For example, using Theorem~\ref{Thrm_gen_CBI_samePKs}, an assessor’s grounds for believing the classifier's failure probability is bounded (i.e., \textbf{PK}\ref{PK1}) remain compatible with the classifier never making, exclusively, either FN or FP errors (i.e. in Fig.~\ref{fig_omega_partition_Ber_1}, the $\Omega$ region has the axes as boundaries). Under this residual possibility, even a single observed FN or FP failure of the software can constitute strong evidence against that failure-type impossibility, and correspondingly strong evidence in favor of substantial unreliability---consistent with the interpretation that the successes observed (not observing the alleged impossible failure mode) are largely due to chance. This impossibility of certain failure modes is ruled out if evidence suggests the classifier's individual FN and FP probabilities cannot be smaller than some non-zero values at the limits of what has been demonstrated to be practically achievable (e.g. based on demonstrated past performance of classifiers using similar algorithms/methods in similar applications, so $\Omega$'s boundaries do not coincide with axes in Fig.~\ref{fig_omega_partition_gen_CBI_simple}).

Using CBI, an assessor can question the strength of reliability evidence---a form of principled ``belief refinement'' in the assessment process. An assessor specifies increasingly more detailed beliefs as more evidence becomes available, with CBI constantly providing the most critical interpretation of the evidence. Conceptually, as more evidence becomes available, the assessor is narrowing the set of prior distributions in $\mathcal D$ that are consistent with the evidence. In contrast, fully specified prior distributions in traditional Bayesian inference leave no room for belief refinement, other than by observational evidence.

The assessor who uses Theorems~\ref{Thrm_CBI_Bernoulli}, \ref{Thrm_CBI_Bernoulli_diffPKS}, is attempting to be conservative, when faced with uncertainty about which failure types have occurred. Sometimes such attempts \emph{are} conservative, and sometimes they \emph{aren't}. CBI results that do not account for the extra failure modes (Theorems~\ref{Thrm_CBI_Bernoulli}, \ref{Thrm_CBI_Bernoulli_diffPKS}) are not necessarily conservative, compared with CBI results that account for the modes (Theorems~\ref{Thrm_gen_CBI_samePKs}, \ref{Thrm_gen_CBI_diffPKs})---this is consistent with CBI's conservatism guarantee. By design/definition, no prior distribution using the same evidence and statistical model (i.e. likelihood function) can yield a more conservative value for the posterior probability of interest. However, if the evidence or statistical model differs from that used for a CBI solution, the guarantee no longer applies, in line with \cite{Popov2025BlackBoxBayesianSafetyAssessment}'s findings. 

\begin{figure}[!t]
    \centering
	\includegraphics[width=0.9\linewidth]{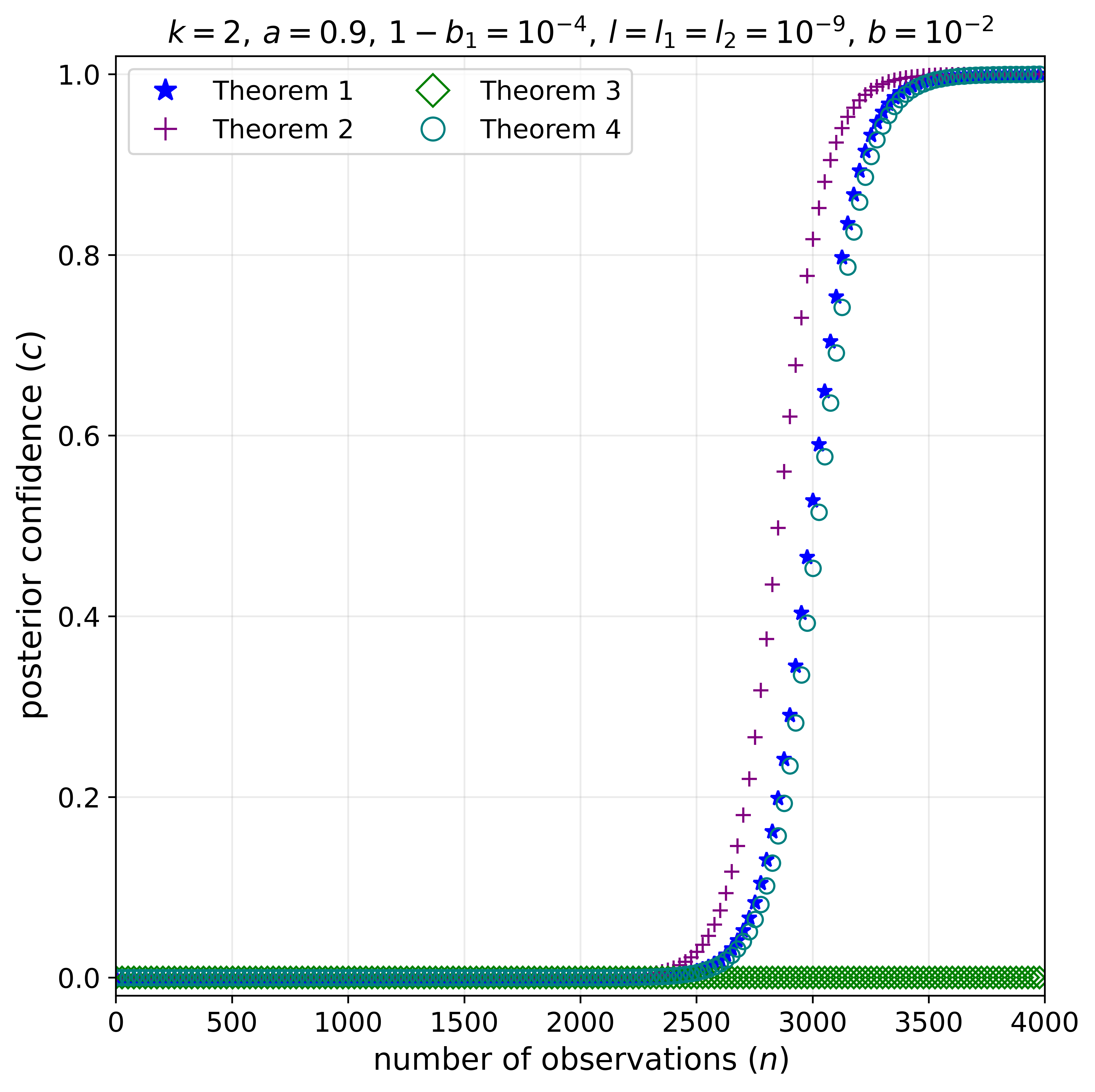}
	\caption{ Conservative posterior confidence guaranteed by  \eqref{eqn_CBIremainsConservativeWithUncertainty}}
	\label{fig_inequality_10}
\end{figure}

In \cite{Popov2025BlackBoxBayesianSafetyAssessment}, Popov compares two Bayesian AV assessment approaches---a ``white box'' approach that accounts for the various operational modes of an AV system, and a ``black box'' approach that uses CBI but does not account for the operational modes. The ``black box'' CBI results were not always conservative compared with the ``white box'' results. The CBI results in \cite{Popov2025BlackBoxBayesianSafetyAssessment} would have been conservative compared with the ``white box'' setting if the operational modes had been accounted for. We can exemplify an analogue in the present setting. Notice how, for a fixed prior, the ratio in \eqref{eqn_Ber_CBI_problem} is a convex
combination of the corresponding ratios in \eqref{eqn_gen_CBI_problem_samePKs}:
% \begin{align*}
% &\ \frac{\mathbb E[L(n,k; P, Q)\mathbf{1}_{P+Q\leqslant b}]}{\mathbb E[L(n,k; P, Q)]}\allowdisplaybreaks\\
% =&\ \sum_{k_1=0}^k w_{k_1}\frac{\mathbb E[L(n,k_1,k-k_1; P, Q)\mathbf{1}_{P+Q\leqslant b}]}{\mathbb E[L(n,k_1,k-k_1; P, Q)]} 
% \end{align*}
% where $w_{k_1}:=\frac{\binom{k}{k_1}\,\mathbb E[L(n,k_1,k-k_1;P,Q)]}
% {\sum_{j=0}^k \binom{k}{j}\,\mathbb E[L(n,j,k-j;P,Q)]}$. Consequently, Theorem~\ref{Thrm_CBI_Bernoulli}'s infimum can always be bounded below by the smallest infima given by Theorem~\ref{Thrm_gen_CBI_samePKs} for some $k_1$. That is,
% \begin{align}
% &\ \ \ \underset{k_1=0,\ldots,k}{\inf}\underset{\mathcal{D}}{\inf}\, \dfrac{\mathbb E[L(n,k_1,k-k_1; P, Q)\mathbf{1}_{P+Q\leqslant b}]}{\mathbb E[L(n,k_1,k-k_1; P, Q)]}\nonumber\allowdisplaybreaks \\
% =&\ \ \ \ \ \ \ \underset{\mathcal{D}}{\inf}\underset{k_1=0,\ldots,k}{\inf}\, \dfrac{\mathbb E[L(n,k_1,k-k_1; P, Q)\mathbf{1}_{P+Q\leqslant b}]}{\mathbb E[L(n,k_1,k-k_1; P, Q)]}\nonumber\allowdisplaybreaks \\
% \leqslant&\ \ \ \ \ \ \ \underset{\mathcal{D}}{\inf}\, \dfrac{\mathbb E[L(n,k; P, Q)\mathbf{1}_{P+Q\leqslant b}]}{\mathbb E[L(n,k; P, Q)]}
% \label{eqn_CBIremainsConservativeWithUncertainty}
% \end{align}
%
%
%
writing $R$ for the ratio in \eqref{eqn_Ber_CBI_problem} and $R_{k_1}$ for the ratio in
\eqref{eqn_gen_CBI_problem_samePKs} with $k_2=k-k_1$, the binomial expansion shows that
$R$ is a convex combination of the $R_{k_1}$. Hence,
$R\geqslant\min\limits_{0\leqslant k_1\leqslant k}R_{k_1}$ and
\begin{equation}
\inf_{\mathcal D}R\geqslant\inf_{\mathcal D}\min_{0\leqslant k_1\leqslant k}R_{k_1}=\min_{0\leqslant k_1\leqslant k}\inf_{\mathcal D}R_{k_1}.
\label{eqn_CBIremainsConservativeWithUncertainty}
\end{equation}

The interchange of infima and minima in \eqref{eqn_CBIremainsConservativeWithUncertainty} is allowed because the priors in $\mathcal D$ and the $k_1$ values are independent, so the iterated infima/minima are an infimum over the Cartesian product $\{0,\ldots,k\}\times\mathcal D$ (see Appendix~\ref{app_C}). An analogous inequality holds between the infima of Theorems~\ref{Thrm_CBI_Bernoulli_diffPKS} and \ref{Thrm_gen_CBI_diffPKs}. In summary, when there is certainty about operational evidence (such as what failure modes occurred), this should be incorporated in the CBI analyses to ensure conservatism. Otherwise, if there \emph{is} uncertainty, using \eqref{eqn_CBIremainsConservativeWithUncertainty} ensures the CBI--based assessment remains conservative.

Inequality \eqref{eqn_CBIremainsConservativeWithUncertainty} is illustrated in Fig.~\ref{fig_inequality_10} with $k=2$ failures. Posterior confidence from Theorems~\ref{Thrm_CBI_Bernoulli}, \ref{Thrm_CBI_Bernoulli_diffPKS}, \ref{Thrm_gen_CBI_diffPKs} grows as the number of failure-free classifications grows. While, confidence from Theorem \ref{Thrm_gen_CBI_samePKs} remains at 0. Due to inequality \eqref{eqn_CBIremainsConservativeWithUncertainty}, at each $n$, Theorem~\ref{Thrm_CBI_Bernoulli}'s confidence curve is bounded below by ``\emph{no confidence}'' from Theorem~\ref{Thrm_gen_CBI_samePKs}. Similarly, Theorem~\ref{Thrm_CBI_Bernoulli_diffPKS}'s curve is bounded below by the minimum confidence from Theorem~\ref{Thrm_gen_CBI_diffPKs} under three pairs of $(k_1,k_2)$ values: $(0,2)$, $(1,1)$, $(2,0)$. 

Fig.~\ref{fig_inequality_10} also illustrates the general ordering: for each $n>2$, Theorem~\ref{Thrm_CBI_Bernoulli_diffPKS} gives the least conservative confidence, then Theorem~\ref{Thrm_CBI_Bernoulli}, then the minimum of Theorem~\ref{Thrm_gen_CBI_diffPKs}'s confidence values, and finally the minimum of Theorem~\ref{Thrm_gen_CBI_samePKs}'s confidence values.

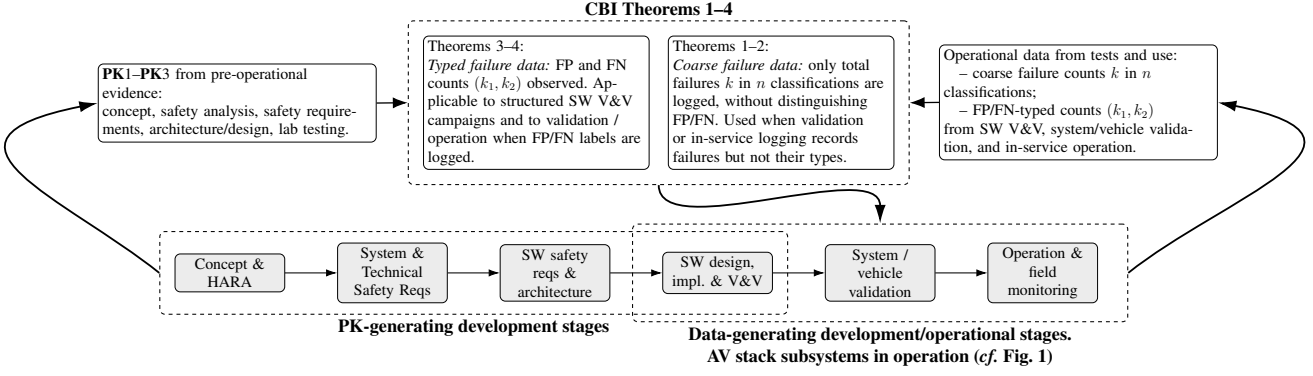
\begin{figure*}[t!]
\centering

\scalebox{0.87}{%
\begin{tikzpicture}[
  font=\footnotesize,
  stage/.style={
    draw,
    rounded corners,
    minimum height=7mm,
    text width=22mm,
    align=center,
    fill=gray!15,
    inner sep=0.8mm,
    font=\footnotesize\linespread{0.95}\selectfont
  },
  theobox/.style={
    draw,
    rounded corners,
    align=left,
    text width=47mm,
    minimum height=8mm,
    font=\footnotesize\linespread{0.93}\selectfont,
    fill=white,
    inner sep=1mm,
    execute at begin node={
      \begingroup
      \hyphenpenalty=10000\relax
      \exhyphenpenalty=10000\relax
    },
    execute at end node={
      \endgroup
    }
  },
  databox/.style={
    draw,
    rounded corners,
    align=left,
    text width=37mm,
    minimum height=8mm,
    font=\footnotesize\linespread{0.93}\selectfont,
    fill=white,
    inner sep=1mm,
    execute at begin node={
      \begingroup
      \hyphenpenalty=10000\relax
      \exhyphenpenalty=10000\relax
    },
    execute at end node={
      \endgroup
    }
  },
  groupbox/.style={
    draw,
    dashed,
    rounded corners,
    line width=0.8pt,
    fill opacity=0,
    text opacity=1
  },
  serialarrow/.style={
    -{Latex[length=1.8mm,width=1.3mm]},
    line width=0.7pt
  },
  grouparrow/.style={
    -{Latex[length=2.3mm,width=1.7mm]},
    line width=0.85pt
  },
  sidearrow/.style={
    -{Latex[length=1.8mm,width=1.3mm]},
    line width=0.8pt
  },
  grouplabel/.style={
    font=\scriptsize\bfseries,
    align=center,
    inner sep=0pt
  }
]

% Layers
\pgfdeclarelayer{bg}
\pgfsetlayers{bg,main}

%--------------------------------------------------------------------
% Upper row: PK evidence, theorem boxes, operational data
%--------------------------------------------------------------------
\node[theobox] (t34)
  {Theorems~\ref{Thrm_gen_CBI_samePKs}--\ref{Thrm_gen_CBI_diffPKs}:\\
   \emph{Typed failure data:} FP and FN counts $(k_1,k_2)$ observed.
   Applicable to structured SW V\&V campaigns and to validation /
   operation when FP/FN labels are logged.};

\node[theobox, right=2mm of t34] (t12)
  {Theorems~\ref{Thrm_CBI_Bernoulli}--\ref{Thrm_CBI_Bernoulli_diffPKS}:\\
   \emph{Coarse failure data:} only total failures $k$ in
   $n$ classifications are logged, without distinguishing FP/FN.
   Used when validation or in-service logging records failures but
   not their types.};

\node[databox, left=5mm of t34] (pkbox)
  {\textbf{PK}\ref{PK1}--\textbf{PK}\ref{PK5} from pre-operational evidence:\\
   concept, safety analysis, safety requirements,
   architecture/design, lab testing.};

\node[databox, right=5mm of t12] (opdata)
  {Operational data from tests and use:\\
   \quad-- coarse failure counts $k$ in $n$ classifications;\\
   \quad-- FP/FN-typed counts $(k_1,k_2)$\\
   from SW V\&V, system/vehicle validation, and in-service operation.};

\node[
  draw=none,
  inner sep=0pt,
  fit=(pkbox) (t34) (t12) (opdata)
] (toprow) {};

%--------------------------------------------------------------------
% ISO 26262-style software lifecycle stages
%--------------------------------------------------------------------
\node[stage, anchor=north west] (concept)
  at ([xshift=11mm,yshift=-6.5mm]toprow.south west)
  {Concept \&\\HARA};

\node[stage, right=4mm of concept] (sysreq)
  {System \&\\Technical\\Safety Reqs};

\node[stage, right=4mm of sysreq] (swreq)
  {SW safety\\reqs \&\\architecture};

\node[stage, right=4mm of swreq] (swdev)
  {SW design,\\impl. \& V\&V};

\node[stage, right=4mm of swdev] (val)
  {System /\\vehicle\\validation};

\node[stage, right=4mm of val] (ops)
  {Operation \&\\field monitoring};

%--------------------------------------------------------------------
% Dashed group boxes
%--------------------------------------------------------------------
\begin{pgfonlayer}{bg}

  \node[
    groupbox,
    inner xsep=1.4mm,
    inner ysep=1.2mm,
    fit=(concept) (sysreq) (swreq) (swdev)
  ] (pkstages) {};

  \node[
    groupbox,
    inner xsep=3mm,
    inner ysep=2.7mm,
    fit=(swdev) (val) (ops)
  ] (datastages) {};

  \node[
    groupbox,
    inner xsep=1.5mm,
    inner ysep=1.5mm,
    fit=(t34) (t12)
  ] (theogroup) {};

\end{pgfonlayer}

%--------------------------------------------------------------------
% Lifecycle arrows
%--------------------------------------------------------------------
\draw[serialarrow] (concept.east) -- (sysreq.west);
\draw[serialarrow] (sysreq.east)  -- (swreq.west);
\draw[serialarrow] (swreq.east)   -- (swdev.west);
\draw[serialarrow] (swdev.east)   -- (val.west);
\draw[serialarrow] (val.east)     -- (ops.west);

%--------------------------------------------------------------------
% Side evidence/data boxes -> theorem group
%--------------------------------------------------------------------
\draw[sidearrow]
  (pkbox.east) -- (theogroup.west);

\draw[sidearrow]
  (opdata.west) -- (theogroup.east);

%--------------------------------------------------------------------
% PK-generating stages -> PK evidence box
%--------------------------------------------------------------------
\coordinate (pkout) at
  (pkbox.south |- pkstages.north);

\draw[grouparrow]
  (pkout) -- (pkbox.south);

%--------------------------------------------------------------------
% Data-generating stages -> operational-data box
%
% Tail begins at the RIGHT EDGE of the lower dashed group.
% The path first moves horizontally outside the group, then vertically
% upward. The arrowhead terminates toward the RIGHT side of the base
% of the operational-data box.
%--------------------------------------------------------------------
\coordinate (dataout) at
  (datastages.east);

\coordinate (datatarget) at
  ([xshift=12mm]opdata.south);

\coordinate (dataelbow) at
  (datatarget |- dataout);

\draw[grouparrow]
  (dataout) --
  (dataelbow) --
  (datatarget);

%--------------------------------------------------------------------
% CBI theorem group -> data-generating stages
%--------------------------------------------------------------------
\coordinate (theoout) at
  (val.north |- theogroup.south);

\coordinate (theoin) at
  (val.north |- datastages.north);

\draw[grouparrow]
  (theoout) -- (theoin);

%--------------------------------------------------------------------
% Labels BELOW the lower dashed boxes
%--------------------------------------------------------------------
\node[
  grouplabel,
  anchor=north
]
  at ([xshift=-8mm,yshift=-1.4mm]pkstages.south)
  {\textbf{PK-generating development stages}};

\node[
  grouplabel,
  anchor=north,
  text width=78mm
]
  at ([xshift=10mm,yshift=-1.4mm]datastages.south)
  {\textbf{Data-generating development/operational stages.}\\
   \textbf{AV stack subsystems in operation
   (\emph{cf.} Fig.~\ref{fig:av_pipeline_classifiers})}};

% The theorem-group title remains on its upper boundary.
\node[
  font=\scriptsize\bfseries,
  fill=white,
  inner xsep=0.8mm,
  inner ysep=0.2mm
]
  at (theogroup.north)
  {\textbf{CBI Theorems~\ref{Thrm_CBI_Bernoulli}--\ref{Thrm_gen_CBI_diffPKs}}};

\end{tikzpicture}%
}

\vspace{-2mm}

\caption{  Illustration of where the CBI theorems apply across an
  ISO~26262-style software (SW) lifecycle (illustrated for the AV stack of Fig.~\ref{fig:av_pipeline_classifiers}). 
  % \textbf{PK}\ref{PK1}--\textbf{PK}\ref{PK5} are
  % justified from pre-operational evidence (e.g. concept, safety analyses, past operational data from similar software used in similar contexts,
  % SW safety requirements, SW architecture/design/algorithms, SW V\&V). During
  % SW V\&V, system/vehicle validation, and in-service operation,
  % operational data---either coarse failure counts or FP/FN-typed
  % counts---are combined with \textbf{PK}s via
  % Theorems~\ref{Thrm_CBI_Bernoulli}--\ref{Thrm_gen_CBI_diffPKs} to obtain conservative bounds on the classifier's
  % failure probability.
  }
\label{fig:practicalapplicationofTheorems}
\end{figure*}

\subsection{Practical Context and Guidance}
\label{subsec_practicalcontext}
%Theorems~\ref{Thrm_CBI_Bernoulli}--\ref{Thrm_gen_CBI_diffPKs} can be used in four ways:
%\begin{enumerate}
%\item[\textbf{3)}] The robustness of claims based on the Theorems can be checked; e.g., the impact of not accounting for failure types when using Theorem~\ref{Thrm_CBI_Bernoulli} is estimated by comparing with Theorem~\ref{Thrm_CBI_Bernoulli_diffPKS} (see Sections~\ref{sec_applications}, \ref{sec_discussion}). For analogous robustness checks, also see \cite{SalakoZhao2023TSE,SalakoZhao2024QRE}.
%\item[\textbf{4)}] The theorems indicate the ``strength'' of evidence needed to support a reliability claim. An example is stipulating how many failure--free classifications need to be observed, or how confident the assessor needs to be in the classifier's accuracy being sufficiently high \emph{before} observing the classifier in operation. In scenario~2, Section~\ref{sec_applications}: under Theorem~3 (which relies on \textbf{PK}\ref{PK1}), a single observed failure drives the worst--case posterior confidence to zero, and no amount of further operational evidence can regain the lost confidence. This ``infinite operational evidence'' phenomenon is a consequence of the admissible--prior set under \textbf{PK}\ref{PK1}; it highlights how conservative claims cannot be justified (they do not exclude requiring ``infinite operational evidence'') unless sufficiently strong prior evidence supports \textbf{PK}s that rule out adversarial corner cases (as \textbf{PK}\ref{PK5} does in Theorem~\ref{Thrm_gen_CBI_diffPKs}).
%\end{enumerate}

Theorems~\ref{Thrm_CBI_Bernoulli}--\ref{Thrm_gen_CBI_diffPKs} can be applied across the classifier's development, assurance review, and deployment (see Fig.~\ref{fig:practicalapplicationofTheorems}) using the following workflow: define the operational slice and reliability claim, determine data fidelity, elicit PK constraints, select the applicable theorem, compute the conservative bound, and perform sensitivity analysis. First, the theorems can be applied during software design, implementation, and V\&V, including unit and integration testing. Second, at system validation and pre-deployment trials (including scenario-based testing and shadow-mode evaluations), the theorems can provide ODD-specific conservative bounds from operational-test data (\emph{cf.} the CBI application in \cite{zhao_assessing_2020}). Third, during early staged roll-out, theorem-based updates support ``confidence bootstrapping'' from failure-free exposure under controlled fleet increments \cite{BishopPovyakaloStrigini2022Bootstrapping}. Finally, during post-deployment monitoring and change management (e.g. after software updates or distribution shift), assessments can be re-run incrementally as evidence accumulates. Across these stages, CBI-supported conservative classifier-reliability claims can serve as quantitative sub-claims in claim--argument--evidence assurance structures \cite{BloomfieldRushby2020A2Manifesto}. Such claims can support higher-level AV safety claims alongside other evidence and complement process evidence with an auditable account of what the data justify \cite{Hawkins2013AssuranceCases}, consistent with Assurance~2.0 \cite{BloomfieldRushby2022A2Confidence}.

Sensitivity analysis can reveal whether CBI results are sensitive to \textbf{PK}-parameter values in a given practical scenario; if they are insensitive, the precise parameter values are unimportant over the specified range---any values from this range ``\emph{will do}'' for assessment, as the following two examples show.

\begin{figure}[t!]
	\centering
	\includegraphics[width=.9\linewidth]{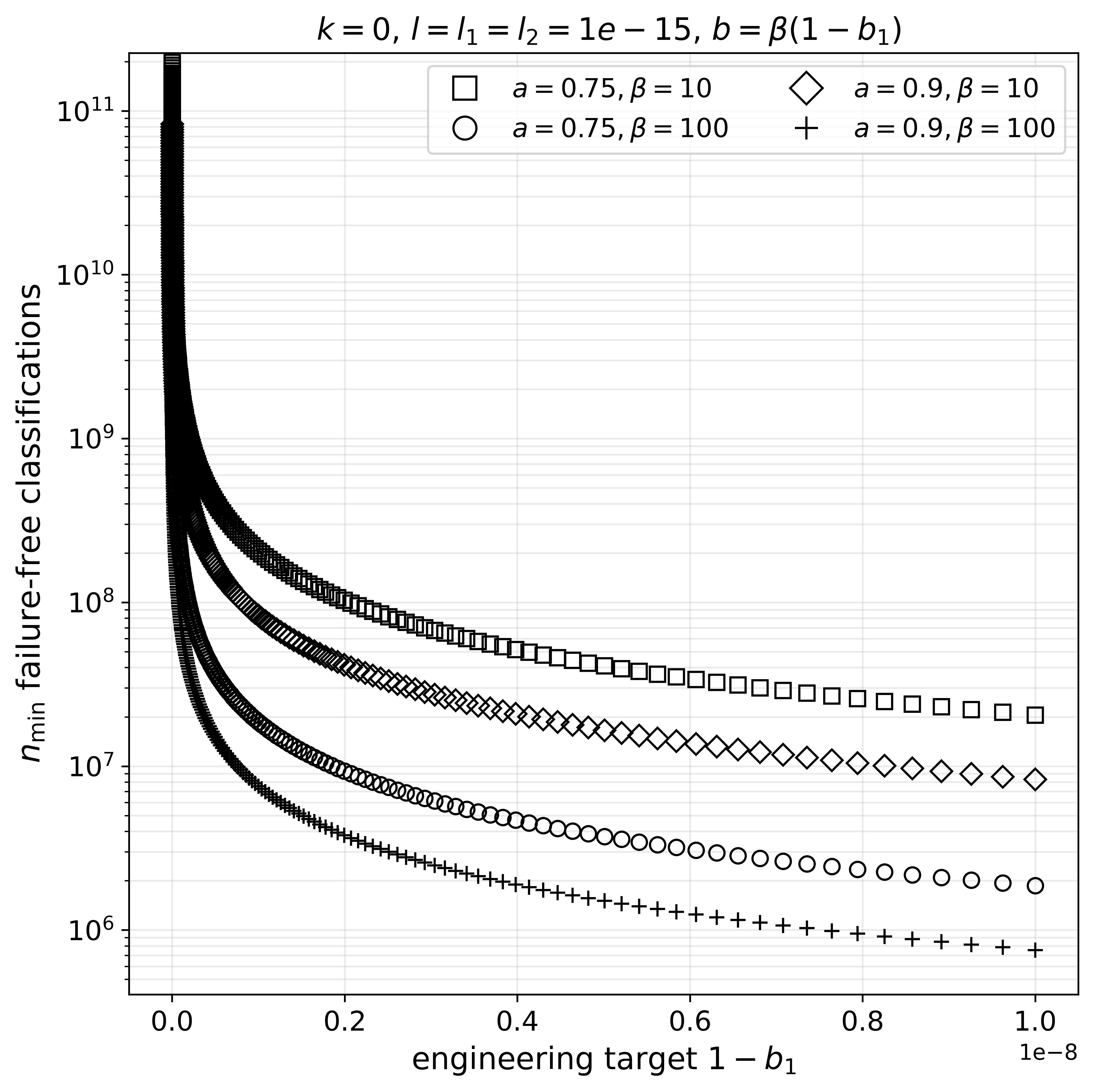}
	\caption{sensitivity of failure-free evidence needed to obtain $95$\% confidence in a \emph{pfc} bound, when the assessor is unsure about the precise value of their confidence $a$ in an engineering target-\emph{pfc} $1-b_1$.}
    \label{fig:a_sensitivity analysis}
\end{figure}
Let $n_{\min}$ be the minimum number of failure-free classifications needed to conservatively achieve $95$\% posterior confidence in the classifier satisfying bound $b$, given the assessor's \textbf{PK}\ref{PK4} confidence $\mathbb P(\Theta\geqslant b_1)=a$ where $b=\beta (1-b_1)$ for some bounded scaling $\beta>1$. The sensitivity of $n_{min}$ to $a,\, b,\, 1-b_1,\, \beta$ can be analysed using families of curves such as those in Fig.~\ref{fig:a_sensitivity analysis}. Moving upward along a vertical line at a given engineering target-\emph{pfc} $1-b_1$, either a larger $n_{\min}$ is needed to gain posterior confidence in smaller $b$ bounds (with a fixed prior confidence ``$a$'' in the classifier satisfying $1-b_1$ bound) or a larger $n_{\min}$ is needed for posterior confidence in a fixed $b$ bound (but with less prior confidence ``$a$'' in the classifier satisfying the  $1-b_1$ bound). So, for a testing budget of $N$ classification tasks, and with an elicited \textbf{PK}\ref{PK4}, the assessor can determine the smallest bound $b^*$ they can achieve $95$\% posterior confidence in. Specifically $b^*=\beta^*(1-b_1^*)$, where $\beta^*$ is either the infimum of those $\beta$ with $(a,\beta)$-curves that intersect the horizontal $N$ budget line at/before $1-b_1$, or the supremum of those $\beta$ with $(a,\beta)$-curves that intersect the budget line at/after $1-b_1$, and $(1-b_1^*)$ is the engineering target at which the $(a,\beta^*)$-curve intersects the budget line.

The i.i.d. categorical process of Section~\ref{sec_methodology} assumes temporal statistical independence and stationarity of failure/success probabilities---properties that may only hold approximately in practice. Sensitivity analysis can reveal whether the i.i.d. assumption leads to relatively optimistic assessment claims compared with using a more sophisticated model (e.g. see \cite{SalakoZhao2023TSE,SalakoZhao2024QRE}). For example, conservative confidence bounds can be obtained using the stationary Markov-chain in Fig.~\ref{fig:three_state_stationary_markov} and these \textbf{PK}s: for $j=1,2,3$, (\textbf{PK}2-a) $\mathbb P({\widetilde \Theta}_j\geqslant b_1)=a$; (\textbf{PK}3-a) $\mathbb P({\widetilde P}_j\geqslant l_1, {\widetilde Q}_j\geqslant l_2)=1$; (\textbf{PK}\ref{PK4}); (\textbf{PK}\ref{PK5}); (\textbf{PK}4)  confidence in the i.i.d. assumption $\mathbb P(P={\widetilde P}_2={\widetilde P}_3, Q={\widetilde Q}_2={\widetilde Q}_3)=\eta$ for some $0\leqslant\eta\leqslant 1$. Thus, $1-\eta$ represents the assessor's doubt that the i.i.d. assumption holds: for $0<\eta<1$, the assessor remains uncertain whether it holds, $\eta=1$ recovers the i.i.d. model, while $\eta=0$ rules it out completely. For $n$ failure-free classifications, the likelihood is
$\Theta{\widetilde \Theta}_1^{\,n-1}$. Hence, the conservative posterior confidence in a \emph{pfc} bound $b$ is $\inf_{\mathcal D}
\frac{
\mathbb E[\Theta{\widetilde \Theta}_1^{\,n-1}1_{P+Q\leqslant b}]
}{
\mathbb E[\Theta{\widetilde \Theta}_1^{\,n-1}]
}$,
subject to the nine preceding \textbf{PK}s. 

\begin{figure}[t]
\centering

\begin{tikzpicture}[
    font=\footnotesize,
    state/.style={
        circle,
        draw=none,
        fill=gray!25,
        text=black,
        minimum size=11.5mm,
        inner sep=0pt
    },
    trans/.style={
        draw=gray!70,
        line width=0.9pt,
        -{Stealth[length=2.4mm,width=1.7mm]},
        shorten <=3pt,
        shorten >=3pt
    },
    lab/.style={
        text=black,
        fill=white,
        inner sep=0.5pt
    }
]

% Compact triangular layout
\node[state] (S)  at (0,1.80) {success};
\node[state] (FP) at (-2.00,0) {FP};
\node[state] (FN) at ( 2.00,0) {FN};

% Self-transitions
\draw[trans]
    (S) edge[loop above, looseness=7]
    node[lab, above] {${\widetilde \theta}_1$} (S);

\draw[trans]
    (FP) edge[loop left, looseness=7]
    node[lab, left] {${\widetilde p}_2$} (FP);

\draw[trans]
    (FN) edge[loop right, looseness=7]
    node[lab, right] {${\widetilde q}_3$} (FN);

% Success <-> FP
\draw[trans]
    (S) edge[bend left=16]
    node[lab, right=3pt] {${\widetilde p}_1$} (FP);

\draw[trans]
    (FP) edge[bend left=16]
    node[lab, left=3pt] {${\widetilde \theta}_2$} (S);

% Success <-> FN
\draw[trans]
    (S) edge[bend right=16]
    node[lab, left=3pt] {${\widetilde q}_1$} (FN);

\draw[trans]
    (FN) edge[bend right=16]
    node[lab, right=3pt] {${\widetilde \theta}_3$} (S);

% FP <-> FN
\draw[trans]
    (FP) edge[bend right=14]
    node[lab, below=1pt] {${\widetilde q}_2$} (FN);

\draw[trans]
    (FN) edge[bend right=14]
    node[lab, above=1pt] {${\widetilde p}_3$} (FP);

\end{tikzpicture}

\vspace{-1mm}

\caption{Three-state stationary Markov model for classifier outcomes. The initial distribution is $(\theta,p,q)$ and the transition probabilities are $({\widetilde \theta}_j,{\widetilde p}_j,{\widetilde q}_j)$ for $j=1,2,3$, satisfying $\theta=1-p-q$, $\ {\widetilde \theta}_j=1-{\widetilde p}_j-{\widetilde q}_j$, $\ {\widetilde p}_1\theta=p-p{\widetilde p}_2-q{\widetilde p}_3$ and $\ {\widetilde q}_1\theta= q-p{\widetilde q}_2-q{\widetilde q}_3$.}
\label{fig:three_state_stationary_markov}
\end{figure}
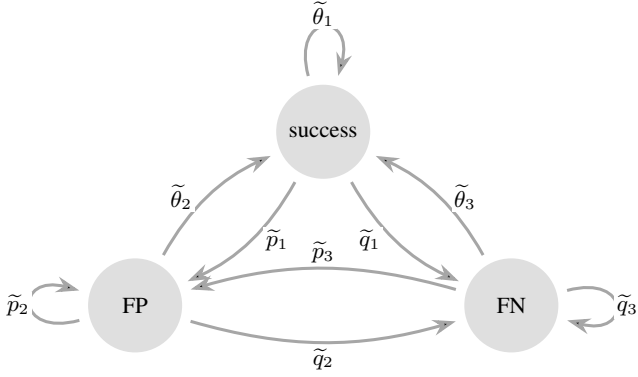
\begin{figure}[t]
\centering
\includegraphics[width=.9\columnwidth]{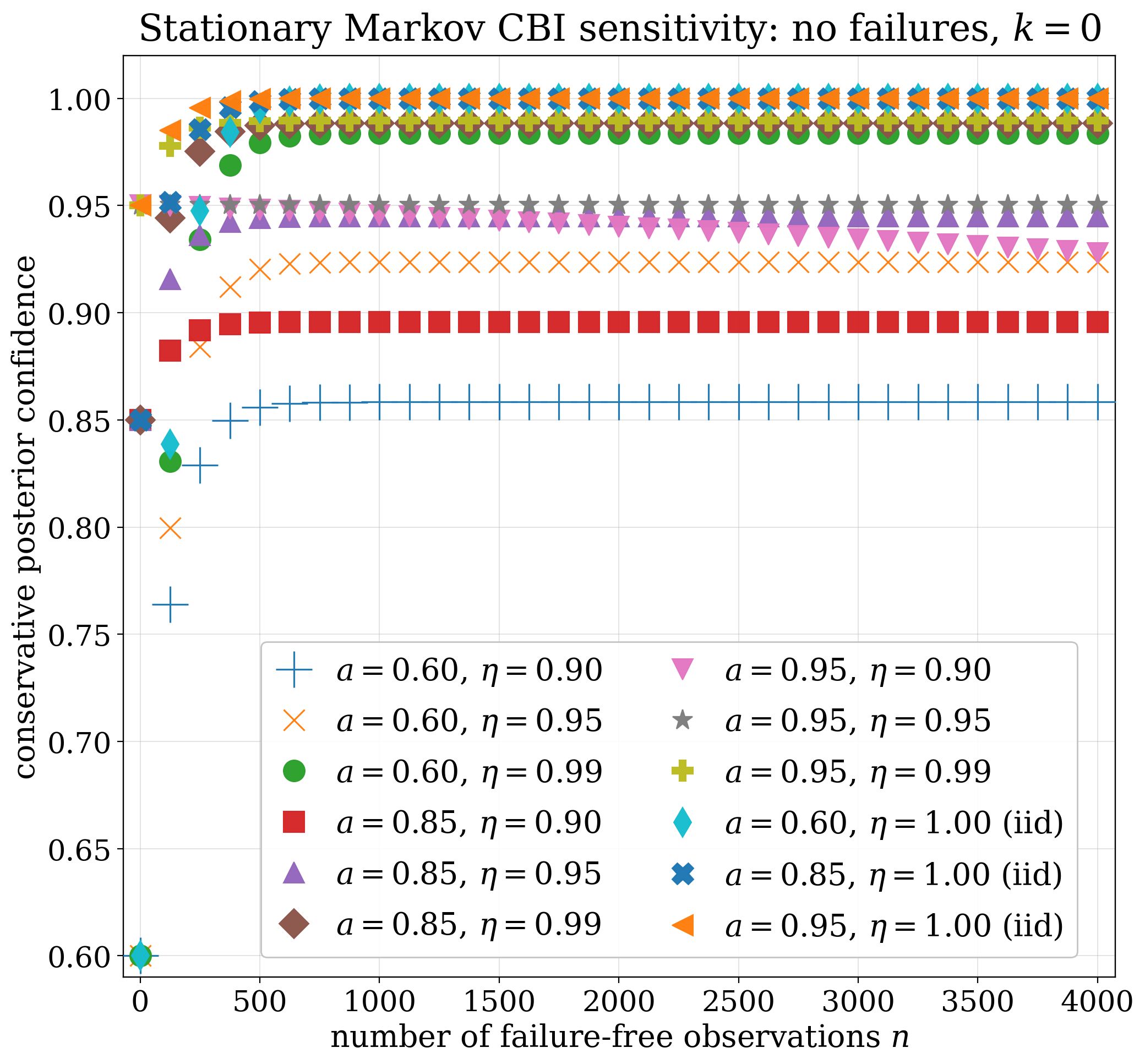}
\caption{Sensitivity of conservative posterior confidence to $a$ and $\eta$, when there are no failures. Here, $l_1=l_2=10^{-9}$, $b=10^{-2}$, $1-b_1=10^{-4}$.}
\label{fig:sensitivity-agamma-k0}
\end{figure}
%
%
%
%
%
% The infimum is solved using the fixed-point framework of~\cite{salako2026fixedpointcharacterisationsextremaldistributions}. There are seven distinct cases for the extremal prior distribution's support locations---these are determined by extremizing Dinkelbach differences over the PK-induced partition of the admissible priors' shared domain. The resulting linear-fractional prior-mass problem is solved by exhaustive vertex enumeration of the feasible mass polytope, with a Charnes--Cooper numerical cross-check; the global infima are plotted in Fig.~\ref{fig:sensitivity-agamma-k0}.

This infimum is solved using the framework of~\cite{salako2026fixedpointcharacterisationsextremaldistributions}: the PKs induce a partition of the feasible Markov parameter space into $16$ ``regions'', over which Dinkelbach minimization yields $7$ distinct extremal prior support-location cases. The resulting linear-fractional mass problem is solved by a Charnes--Cooper LP over a feasible mass polytope; the objective minimum occurs at a vertex. These global infima are plotted in Fig.~\ref{fig:sensitivity-agamma-k0}.

Fig.~\ref{fig:sensitivity-agamma-k0} compares the conservative posterior confidence in bound $b=10^{-2}$, from i.i.d. and stationary non-i.i.d. models. In this scenario, the i.i.d. model is the most optimistic; the greater the doubt in the i.i.d. assumption (i.e. the smaller $\eta$ is), the smaller the worst-case confidence in the classifier. By accumulating failure-free evidence, the posterior confidence from \emph{all} of the models initially monotonically increases, but the non-i.i.d. curves eventually monotonically decrease at different rates.  
Asymptotically, all non-i.i.d. curves converge to $0$ because long failure-free runs can be relatively likely even when the \emph{pfc} exceeds $b$; e.g. persistent operating conditions can induce long benign stretches favoring success. Here, the closer $\eta$ and $a$ are to $1$, the better the agreement between the i.i.d. and non-i.i.d. models for longer. The curves highlight when the i.i.d. model does not deviate substantially from the non-i.i.d. models. This is consistent with the findings in \cite{SalakoZhao2023TSE}.         

\subsection{Limitations and Future Work}
\label{sec:limitations_future}
%, used to estimate the impact of the i.i.d. assumption on assessments, exemplifies a general strategy. Whenever there is reason to  doubt a model's likelihood function (e.g. evidence suggests certain modeling assumptions may not hold in a given practical context, or important information is missing from the available data, or there are latent variables), such uncertainty can be captured by the prior and a generalized likelihood (with the ``likelihood-in-question'' as a special case). Using the generalized model, the worst claim supported by this model can be determined and compared against the worst claim supported by the more restrictive model. If the difference between these claims is negligible, then there is little impact from using the more restrictive model conservatively.

AV operational data can be clustered by route/weather, can exhibit non--stationarity (e.g. software updates, distribution shift from changing driving conditions), and may contain correlated failures (\emph{cf.} safety-critical software in general \cite{Bishop_failureclustering_1993}).

The statistical model treats outcomes as \emph{success} \emph{vs.} \emph{two failure modes} (FP/FN), and explicitly collapses multiple success types. This ``failure-centric'' focus omits potentially decision--relevant structure (e.g. near--misses, degraded performance modes, or graded severity). More broadly, ``low fidelity'' in \eqref{eqn_Bernoulli_CBI_problem} is primarily \emph{missing failure--type labels}, whereas operational fidelity issues can also include noisy labels, uncertain counts, partial observability of reliability-related events/event--types, missing details of operational modes, more failure and success types, and scenario--dependent misclassification costs.

The primary quantity \emph{pfc} (i.e. $P+Q$) is bounded in the form of confidence statements (e.g. \eqref{eqn_upperconfbnd_accuracy}--\eqref{eqn_Bernoulli_CBI_problem}). However, system-level safety cases often require linking such component-level rates to hazard rates, exposure, and severity (e.g. how ``classifications'' map to miles, ODD coverage, or risk per hour). The AV examples in Sections~\ref{sec_applications}, \ref{sec_discussion} do not yet provide calibrated mappings from \emph{pfc} bounds to end-to-end safety metrics.

The  $(P,Q)$ distribution can change due to changes in the operational environment---e.g. the AV operating in  driving conditions very different from those during assessment---or due to significant AV software updates---e.g. further classifier training or a change in decision thresholds. Incorporating such ``distribution drift'' into conservative assessments can be done; e.g. using generalizations of CBI models that account for changes in operational context, such as the models used in \cite{zhao_assessing_2020,salako_conservative_2021,AghazadehChakherlouStrigini2024PreChange}. Without such generalizations, this paper's results must be applied only ``in-between'' significant software updates or ``in-between'' changes in operational environment.

Future work can investigate supporting other forms of conservative dependability claims; e.g. CBI bounds on the timeliness of mitigating actions taken by safety subsystems.

\section{Conclusion}
\label{sec_conc}

The fidelity of operational evidence places limits on what claims are justified in an assessment of safety-critical-software reliability. When assessing safety-relevant classifiers in an AV stack, this paper illustrates how an assessor's attempts at supporting conservative claims---using ``low-fidelity'' evidence with a ``\emph{guaranteed-to-be-conservative}'' statistical inference approach (i.e. CBI)---can be valid sometimes, but invalid at other times. The problem is that the CBI guarantee is conditional on the evidence, so that if some evidence is not used with CBI (such as missing information about failure-types that have occurred) the results of inference may be undesirably optimistic, compared with Bayesian inference based on a fully--specified feasible prior and ``all'' of the evidence.

An assessor often has no choice but to work with ``missing detail'' in operational evidence. The CBI extensions presented in this paper aid in using such evidence conservatively. 
 
%\section*{Acknowledgments}
%We are grateful to the anonymous reviewers whose comments were very helpful in improving the presentation. We are also grateful to Bev Littlewood for giving very helpful feedback on an initial draft of this paper. This work was partly funded by the European Union’s Horizon 2020 Research and Innovation Programme under grant agreement No 956123, and by the UK EPSRC through the End-to-End Conceptual Guarding of Neural Architectures [EP/T026995/1]. Xingyu Zhao's contribution is partially supported through Fellowships at the Assuring Autonomy International Programme.
%We thank the anonymous reviewers whose comments helped us to improve the manuscript.

\bibliographystyle{IEEEtran}
\bibliography{ref}

\appendices
\section{Proof of Theorem~\ref{Thrm_gen_CBI_diffPKs}}
\label{subsec_proof_Theorem4}
\begin{proof}
%\begin{theorem}
%\label{Thrm_gen_CBI_diffPKs}
%\noindent\textbf{Problem:}

%\noindent For $n,\,k_1,\,k_2$, $a$, $b$, $b_1$, $l_1$, $l_2$, $P$, $Q$, $\mathcal{D}$ as already defined and constrained, we seek 
%\begin{align*} &\underset{\mathcal{D}}{\inf}\, \dfrac{\mathbb E[L(n,k_1,k_2; P, Q)\mathbf{1}_{P+Q\leqslant b}]}{\mathbb E[L(n,k_1,k_2; P, Q)]} \\
%    s.t. \,\,&\text{\bf PK}\ref{PK4},\,\text{\bf PK}\ref{PK5}  \nonumber 
%\end{align*}

Define $C=[p+q\leqslant 1-b_1]$ and $G=[p+q\leqslant b]$. The Fig.~\ref{fig_omega_partition_gen_CBI_simple} partition is therefore just ${\boldsymbol S}_1=C^c\cap\Omega$ and ${\boldsymbol S}_2=C\cap\Omega$. So, every admissible prior distribution satisfies ${\mathbb P}({\boldsymbol S}_1)=1-a$, ${\mathbb P}({\boldsymbol S}_2)=a$. There is no remaining freedom in the total masses assigned to ${\boldsymbol S}_1$ and ${\boldsymbol S}_2$.

Under the theorem’s reliability-target ordering, $1-b_1\leqslant b$ so ${\boldsymbol S}_2=C\cap\Omega\subseteq G$. Thus, all mass in ${\boldsymbol S}_2$ contributes to the numerator. Meanwhile, ${\boldsymbol S}_1$ contains points above the $b$-line, because ${\boldsymbol S}_1=[p+q>1-b_1]\cap\Omega$ and $1-b_1\leqslant b$, so ${\boldsymbol S}_1\cap G^c\neq\varnothing$, up to usual boundary/closure conventions.

Now, write $f(p,q)=L(p,q)\mathbf 1_G(p,q)$ and $g(p,q)=L(p,q)$, where $L(p,q)$ is shorthand for $L(n,k_1,k_2;p,q)$. Then, we seek $\inf_{{\mathbb P}\in \mathcal D}\frac{\mathbb E_{\mathbb P}[f(P,Q)]}{\mathbb E_{\mathbb P}[g(P,Q)]}$. The functions are bounded on $\Omega$: $g=L$ is continuous, while $f=L\mathbf 1_G$ is bounded piecewise continuous with a possible discontinuity on $p+q=b$. Hence $f$, $g$ are integrable and Section~\ref{subsec_binaryclass_statsmodel} restricts $\mathcal D$ to priors that give well-defined posterior probabilities, so we may appeal to fixed-point arguments to solve the problem after refining ${\boldsymbol S}_1,{\boldsymbol S}_2$ by the regions $G$, $G^c$, and the boundary $p+q=b$. Indeed,  since these refinements are Borel measurable and partition $\Omega$, by Theorem~1 of \cite{moreno1991robust} or Proposition~2.2 of \cite{salako2026fixedpointcharacterisationsextremaldistributions}, we restrict the optimization to priors that are the weak limits of feasible discrete priors, supported by one selected location in each of ${\boldsymbol S}_1$ and ${\boldsymbol S}_2$. Dinkelbach-type iteration gives the infimum. This problem is ``separable'', involving a weighted sum of Dinkelbach differences that can be minimized separately. 

For a candidate infimum value $\phi$ for the objective function, each Dinkelbach difference takes the form $h_\phi=f-\phi g$. On subset $G$, $h_\phi=(1-\phi)L$, whereas on $G^c$, $h_\phi=-\phi L$. The infimum is given by appropriately allocating probability mass to $(p,q)$ locations in ${\boldsymbol S}_1$ and ${\boldsymbol S}_2$ that minimize $h_\phi$.

Since $0\leqslant \phi\leqslant 1$, this has the following consequences for determining the $h$-minimizing $(p,q)$ locations. On ${\boldsymbol S}_2$, we have ${\boldsymbol S}_2\subseteq G$. Therefore, on ${\boldsymbol S}_2$, minimizing $h_\phi=(1-\phi)L$ means minimizing $L$; i.e. the $(p,q)$ location that gives the likelihood value $L_{2*}
=
\inf_{(p,q)\in {\boldsymbol S}_2\cap G}L(p,q)$ on ${\boldsymbol S}_2$, with boundary locations interpreted as limits of non-boundary locations.

Analogously, on ${\boldsymbol S}_1$, the minimizing $(p,q)$ location can be in ${\boldsymbol S}_1\cap G$ or ${\boldsymbol S}_1\cap G^c$. On ${\boldsymbol S}_1\cap G$, the Dinkelbach difference is non-negative, $h_\phi=(1-\phi)L\geqslant 0$, while on ${\boldsymbol S}_1\cap G^c$ it is non-positive, $h_\phi=-\phi L\leqslant 0$. Thus, the location lies in ${\boldsymbol S}_1\cap G^c$, and there it maximizes $L$: 
%That is, the location that gives the likelihood value 
\[
L_1^*
=
\sup_{(p,q)\in {\boldsymbol S}_1\cap G^c}L(p,q)
=
\sup_{(p,q)\in {\boldsymbol S}_1\cap[p+q>b]}L(p,q).
\]
Again, with boundary locations interpreted as limits of non-boundary locations, the ``$\sup$'' is over ${\boldsymbol S}_1\cap[p+q\geqslant b]$.

These two $h_{\phi}$-minimizing $(p,q)$ locations are the support of an extremal prior that assigns probability $1-a$ to the location in ${\boldsymbol S}_1\cap G^c$ and probability $a$ to the location in ${\boldsymbol S}_2\cap G$; for example, see Fig.~\ref{fig_preferred_locations_gen_CBI_1}. 
%\[
%{\mathbb P}^*=(1-a)\delta_{(p_1,q_1)}+a\delta_{(p_2,q_2)},
%\]
%where \((p_1,q_1)\in {\boldsymbol S}_1\cap G^c\) attains or approaches \(L_1^*\), and \((p_2,q_2)\in {\boldsymbol S}_2\cap G\) attains or approaches \(L_{2*}\).
Using these support values gives
\[
\Phi^*
=
\frac{
a L_{2*}
}{
aL_{2*}+(1-a)L_1^*
}
\mathbf 1_{1-b_1\leqslant b}.
\]

The indicator $\mathbf 1_{1-b_1\leqslant b}$ means the solution is not necessarily zero when \textbf{PK}\ref{PK4}'s high-confidence region lies at or below the $b$-line; i.e. the indicator is $1$ when $1-b\leqslant b_1$.\qedhere 
\end{proof}

\section{Proof that $n_2\overset{n_1\to\infty}{\longrightarrow}\lceil\frac{b_1}{1-b_1}\rceil$ in Fig.~\ref{fig_n2vsn1_comparisons}}
\label{app_B}

\begin{proof}
For each $n_1$ in Scenario~2, deduce a bound
$b=b(n_1)$ by requiring that the failure-free evidence $k=0$
over $n_1$ classifications yields a target posterior confidence
$c$:
\begin{equation}
\Phi_0(b;n_1)=c.
\label{eq:appB-target0}
\end{equation}
After one untyped failure on the $(n_1+1)$-th classification
and $n_2$ further failure-free classifications, the total number of classifications is $n=n_1+1+n_2$, with $k=1$, and $n_2$ is chosen so that the conservative
confidence returns to $c$:
\begin{equation}
\Phi_1(b;n_1+1+n_2)=c.
\label{eq:appB-target1}
\end{equation}

Define $K:=\frac{a(1-c)}{c(1-a)}$. Scenarios~1 and~2 have parameter values that satisfy $a<c$, so $0<K<1$. For $k=0$, in the regime $b>1-b_1$, Theorems~\ref{Thrm_CBI_Bernoulli} and~\ref{Thrm_CBI_Bernoulli_diffPKS} give $\Phi_0(b;n_1)
=
\frac{a b_1^{n_1}}
     {a b_1^{n_1}+(1-a)(1-b)^{n_1}}$. Using \eqref{eq:appB-target0} and rearranging, 
\begin{equation}
\left(\frac{1-b}{b_1}\right)^{n_1}=K.
\label{eq:appB-K}
\end{equation}
Hence $b=1-b_1K^{1/n_1}$, therefore $b\downarrow 1-b_1$ as $n_1\to\infty$.

Write $r:=p+q$. For one untyped failure in $n$ classifications, the likelihood
depends on $(p,q)$ only through $r$; i.e. $L_1(r)=r(1-r)^{n-1}$. Its derivative is $L_1'(r)=(1-r)^{n-2}(1-nr)$,
so $L_1$ is unimodal with its maximum at $r=1/n$.

In $S_2$, the feasible range of $r$ is
$[r_0,1-b_1]$, where $r_0=l$ for Theorem~\ref{Thrm_CBI_Bernoulli} and
$r_0=l_1+l_2$ for Theorem~\ref{Thrm_CBI_Bernoulli_diffPKS}. In Scenario~2,
$0<r_0<1-b_1$. $L_1$ is unimodal, so its minimum over this interval occurs at an endpoint. Moreover, $\frac{L_1(1-b_1)}{L_1(r_0)}
=
\frac{1-b_1}{r_0}
\left(\frac{b_1}{1-r_0}\right)^{n-1}
\longrightarrow 0$ because $b_1/(1-r_0)<1$. Since $n\geqslant n_1+1$, it follows
that, for all sufficiently large $n_1$, the least-favorable
likelihood in $S_2$ is attained at $r=1-b_1$.
Also, since $\frac1n\leqslant\frac1{n_1+1}\longrightarrow0$, whereas $b\to1-b_1>0$, for all sufficiently large $n_1$
we have $b> 1-b_1>1/n$. Thus, $L_1$ is decreasing for $r> b$,
and the most-favorable likelihood (i.e. supremum) over $S_1\cap[r>b]$ is
attained at the boundary $r=b$, with boundary point likelihood values interpreted as limits where necessary. Consequently, for all sufficiently large $n_1$, Theorems~\ref{Thrm_CBI_Bernoulli}
and~\ref{Thrm_CBI_Bernoulli_diffPKS} give $\Phi_1(b;n)
=
\frac{a(1-b_1)b_1^{n-1}}
     {a(1-b_1)b_1^{n-1}
      +(1-a)b(1-b)^{n-1}}$. By \eqref{eq:appB-target1} using
$n-1=n_1+n_2$, 
\begin{equation}
\left(\frac{1-b}{b_1}\right)^{n_1+n_2}
=
K\frac{1-b_1}{b}.
\label{eq:appB-recovery}
\end{equation}
Dividing \eqref{eq:appB-recovery} by
\eqref{eq:appB-K} yields
\begin{equation}
\left(\frac{1-b}{b_1}\right)^{n_2}
=
\frac{1-b_1}{b}.
\label{eq:appB-n2relation}
\end{equation}
Using $b=1-b_1K^{1/n_1}$, 
\eqref{eq:appB-n2relation} gives (for real variable $n_2$)
\begin{equation}
n_2
=
n_1
\frac{\displaystyle
 \ln\!\left(
 \frac{1-b_1}{1-b_1K^{1/n_1}}
 \right)}
{\ln K}.
\label{eq:appB-n2}
\end{equation}

Set $x=1/n_1$.  As $n_1\to\infty$, $x\downarrow0$, and
\eqref{eq:appB-n2} becomes $n_2
= \ln\!\left(
 \frac{1-b_1}{1-b_1K^x}
 \right)/x\ln K$. Both numerator and denominator tend to zero.  Since
$0<K<1$ and $0<b_1<1$, l'H\^opital's rule gives $\lim\limits_{n_1\to\infty}n_2
=
\lim\limits_{x\downarrow0}
\big(\frac{b_1K^x\ln K}{1-b_1K^x}\big)/\ln K 
=
\frac{b_1}{1-b_1}$.
%This proves the stated continuous asymptote.
%Next, justify the asymptotic form of $\Phi_1$ corresponding to the right of the hump in Fig.~\ref{fig_n2vsn1_comparisons}.
%
%
%
% Finally, for fixed $n_1$ (and hence fixed $b$),
% \[
% \Phi_1
% =
% \left[
% 1+
% \frac{(1-a)b}{a(1-b_1)}
% \left(\frac{1-b}{b_1}\right)^{n_1+n_2}
% \right]^{-1}.
% \]
% Since $(1-b)/b_1=K^{1/n_1}<1$, $\Phi_1$ is increasing in
% $n_2$. 
The limiting required number of additional classifications is $\lceil \frac{b_1}{1-b_1}\rceil$.%, where $n_2$ is the real-valued solution in
% \eqref{eq:appB-n2}.
\end{proof}

\section{Proof of Interchange of Infima and Minima in \eqref{eqn_CBIremainsConservativeWithUncertainty}}
\label{app_C}
\begin{proof}
Let $K:=\{0,1,\dots,k\}$ and let $\mathcal{D}$ denote the admissible set of priors. The elements of $K$ and $\mathcal D$
do not depend on each other. Define the objective
$f:K\times\mathcal{D}\to[0,1]$ by letting 
$f(k_1,{\mathbb P})$ be the quantity being minimized in~\eqref{eqn_CBIremainsConservativeWithUncertainty} (the posterior
confidence expression under $(k_1,k-k_1)$ and prior ${\mathbb P}$).
Set
\[
\alpha := \min_{k_1\in K}\ \inf_{{\mathbb P}\in\mathcal{D}} f(k_1,{\mathbb P}),
\qquad
\beta := \inf_{{\mathbb P}\in\mathcal{D}}\ \min_{k_1\in K} f(k_1,{\mathbb P}).
\]
We show $\alpha=\beta$ by proving both inequalities.

\emph{First, $\alpha\leqslant \beta$.}
Fix any ${\mathbb P}\in\mathcal{D}$. For every $k_1\in K$, $\inf_{{\mathbb P}'\in\mathcal{D}} f(k_1,{\mathbb P}') \leqslant f(k_1,{\mathbb P})$. Taking $\min_{k_1\in K}$ of both sides gives $\min_{k_1\in K}\inf_{{\mathbb P}'\in\mathcal{D}} f(k_1,{\mathbb P}') \leqslant \min_{k_1\in K} f(k_1,{\mathbb P})$,
i.e.\ $\alpha \leqslant \min_{k_1\in K} f(k_1,{\mathbb P})$. Since this holds for every
${\mathbb P}\in\mathcal{D}$, taking $\inf_{{\mathbb P}\in\mathcal{D}}$ yields $\alpha \leqslant \inf_{{\mathbb P}\in\mathcal{D}}\min_{k_1\in K} f(k_1,{\mathbb P})=\beta$.

\emph{Second, $\alpha\geqslant \beta$.}
Fix any $k_1\in K$. For every ${\mathbb P}\in\mathcal{D}$, $f(k_1,{\mathbb P}) \geqslant \min_{k_1'\in K} f(k_1',{\mathbb P})$. Taking $\inf_{{\mathbb P}\in\mathcal{D}}$ of both sides gives $\inf_{{\mathbb P}\in\mathcal{D}} f(k_1,{\mathbb P}) \geqslant \inf_{{\mathbb P}\in\mathcal{D}}\min_{k_1'\in K} f(k_1',{\mathbb P})$, i.e.\ $\inf_{{\mathbb P}\in\mathcal{D}} f(k_1,{\mathbb P}) \geqslant \beta$. Since this holds for every
$k_1\in K$, taking $\min_{k_1\in K}$ yields $\alpha=\min_{k_1\in K}\inf_{{\mathbb P}\in\mathcal{D}} f(k_1,{\mathbb P}) \geqslant \beta$.

Combining $\alpha\leqslant\beta$ and $\alpha\geqslant\beta$ gives $\alpha=\beta$.
\end{proof}

\end{document}